\documentclass[11pt,a4paper]{article}

\usepackage[margin=1in]{geometry}
\usepackage[T1]{fontenc}
\usepackage{lmodern}
\usepackage{microtype}
\usepackage[round]{natbib}

\usepackage[hyphens]{url}
\usepackage{graphicx}
\usepackage{booktabs}
\usepackage{amsmath,amssymb,amsthm}
\usepackage{placeins}
\usepackage{array}
\usepackage[unicode,hidelinks]{hyperref}
\hypersetup{
  pdftitle={Grounding Predictive Reliability through Masked Observability},
  pdfauthor={Shang Gao, Bernhard Sick, Franz Götz-Hahn}
}

\usepackage{authblk}

\title{Grounding Predictive Reliability through Masked Observability}
\author[1]{Shang Gao\thanks{ORCID: \href{https://orcid.org/0009-0005-5474-1888}{0009-0005-5474-1888}}}
\author[1]{Bernhard Sick\thanks{ORCID: \href{https://orcid.org/0000-0001-9467-656X}{0000-0001-9467-656X}}}
\author[1]{Franz G\"otz-Hahn\thanks{ORCID: \href{https://orcid.org/0000-0003-3465-5040}{0000-0003-3465-5040}}}
\affil[1]{University of Kassel, Germany}
\date{}

\newcommand{\figthreeyaxis}{Axes share a log scale.}

\newtheorem{proposition}{Proposition}

\newcommand{\method}{MORSE}
\newcommand{\udata}{U_{\mathrm{data}}}
\newcommand{\umean}{U_{\mathrm{mean}}}
\newcommand{\urel}{U_{\mathrm{rel}}}
\newcommand{\upred}{U_{\mathrm{pred}}}
\newcommand{\robs}{R_{\mathrm{obs}}}
\newcommand{\rstar}{R^{\ast}_{\mathrm{obs}}}
\newcommand{\rnorm}{r_{\mathrm{obs}}}
\newcommand{\vref}{V_{\mathrm{ref}}}
\newcommand{\Lmean}{L_{\mathrm{mean}}}
\newcommand{\sg}{\operatorname{sg}}
\newcommand{\E}{\mathbb{E}}
\newcommand{\Var}{\operatorname{Var}}
\newcommand{\Cov}{\operatorname{Cov}}

\begin{document}
\maketitle

\begin{abstract}
An uncertainty estimate is meaningful only if its semantics can be tied to observable predictive behavior. Normal-Inverse-Gamma (NIG) evidential regression provides a structured decomposition of predictive variability into data-scale and predictive-mean components, but marginal likelihood does not identify their allocation. We address this semantic ambiguity in masked representation learning. MORSE (Masked Observability for Reliable Structured Encoding) predicts the same hidden target under controlled changes in visible context and uses normalized cross-context variation to ground predictive-mean uncertainty, while a Gaussian objective learns the conditional mean and data scale. The normalized variation plays the role of the NIG factor $1/\nu$, which the marginal leaves unidentified; this yields a policy-relative observability variance that measures dependence on visible evidence, distinct from variation induced by resampling the training set. On a Gaussian-process benchmark with a contextual reference built from analytic Bayes conditional means, MORSE has the lowest log-variance error among compared evidential, ensemble, and dropout estimators in all nine tested settings, while the estimator closest to training-set variation perturbs the model rather than its context in every setting, so the two references separate the estimators empirically. MORSE amortizes observability estimation into one forward pass per observation state. On a heteroscedastic pulse benchmark, its predictor's cross-context variance approaches a large-sample reference as training data increase. MORSE also improves low-label transfer on PTB-XL and in-distribution trajectory prediction on IMPTC.
\end{abstract}

\section{Introduction}

Two regression models can output the same predictive distribution, meaning the same mean and total variance, and still disagree about why they are uncertain: one blames noise in the target, the other blames what it cannot infer from the evidence it sees. If nothing in the data distinguishes these two accounts, an uncertainty estimate that claims to make the distinction reports a quantity it was never taught to produce. This is the situation in Deep Evidential Regression (DER)~\citep{amini2020deep}, whose Normal-Inverse-Gamma (NIG) head names a data-scale component and a predictive-mean component that the likelihood then collapses: a continuum of allocations yields the same fit (Figure~\ref{fig:motivation}). Recent work has made this concrete: the reported uncertainty is leaky~\citep{meinert2023unreasonable}, unidentifiable~\citep{jurgens2024faithfully}, and often useful only as an out-of-distribution signal rather than a faithful one~\citep{shen2024mirage}. Because the marginal cannot fix the allocation, its meaning has to come from information beyond the likelihood. The missing ingredient is an observable, target-specific behavioral signal for the allocation.

Masked representation learning supplies a controlled way to vary the evidence available about a target. We keep the target hidden, sample different visible contexts, and measure the variance of the resulting predictions with the model held fixed. By the law of total variance, this response is the between-context component of the mask-induced predictive mixture. Dividing it by the fitted local data scale gives a dimensionless response that occupies the same structural position as $\umean/\udata=1/\nu$ in the NIG decomposition. This correspondence provides an observable, policy-relative training target for the factor left unidentified by the marginal likelihood.

\begin{figure*}[t]
    \centering
    \includegraphics[width=\textwidth]{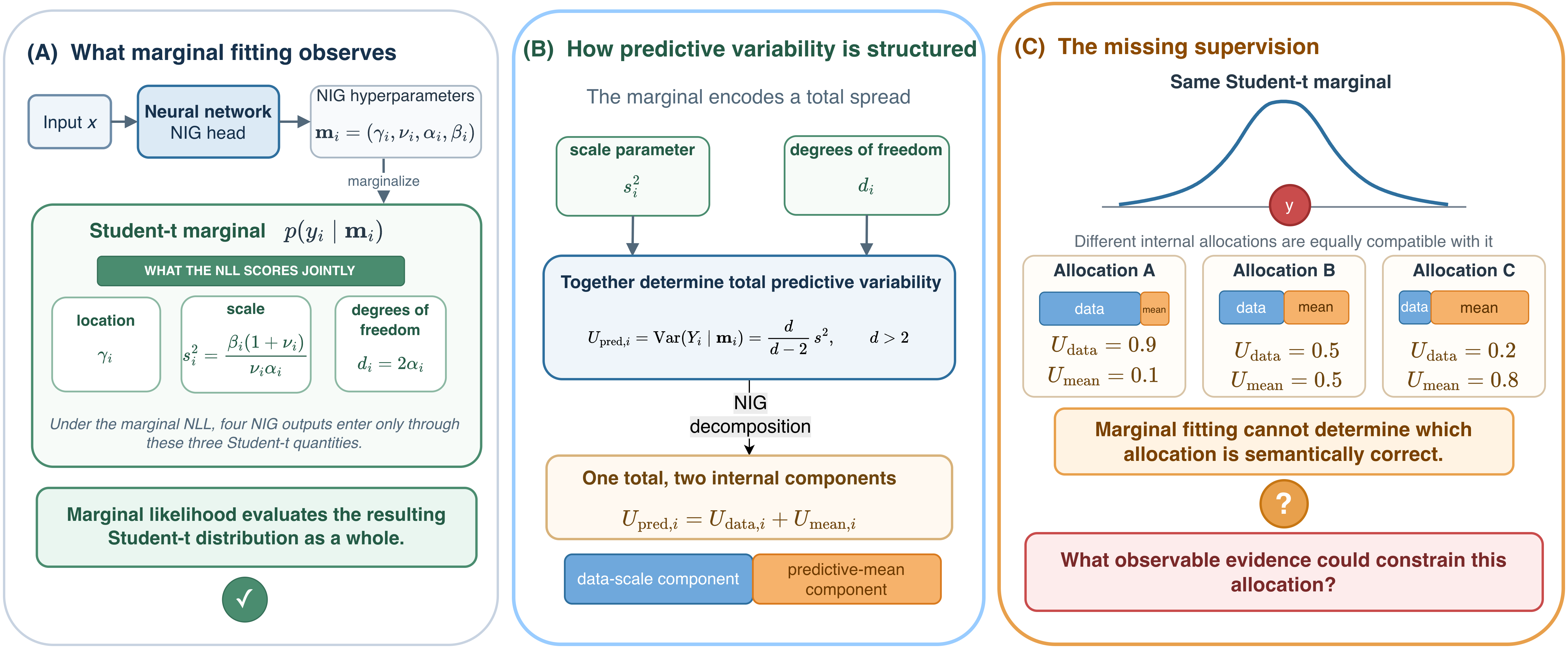}
    \caption{\textbf{Marginal fitting leaves the internal allocation open.} (A) Under standard DER, the four NIG outputs enter the marginal likelihood only through the location, scale, and degrees of freedom of a Student-$t$ distribution. (B) Its total predictive variance can be written as a data-scale component plus a predictive-mean component. (C) The same marginal is compatible with different internal allocations, so fitting it cannot determine which allocation is appropriate for a target. MORSE supplies the missing evidence through controlled changes in visible context.}
    \label{fig:motivation}
\end{figure*}

MORSE (Masked Observability for Reliable Structured Encoding) learns this target from $K$ masked views that share the same hidden positions. A Gaussian objective fits the predictive mean and data scale, while magnitude and ranking losses supervise $1/\nu$ using the detached cross-context response. This separates their direct output-level supervision paths. The score is trained to estimate context sensitivity under the chosen masking policy. We compare it with contextual and training-set references and test the pretrained representations on downstream tasks.

Our contributions are:
\begin{itemize}
    \item We observe that the normalized between-context variance of a mask-induced predictive mixture occupies the position of the NIG factor $1/\nu$, the coordinate the marginal likelihood leaves free, and adopt this correspondence as the definition of what $1/\nu$ reports. This correspondence gives the allocation a policy-relative observable target and gives $\nu$ an observable training target as an effective count of the evidence the policy varies; we also characterize the population limit of the target (Section~\ref{sec:theory}).
    \item We introduce MORSE, a common-target $K$-view objective that supervises this factor from context sensitivity, separates it from the direct residual path, and amortizes the multi-view response into one forward pass per observation state (Section~\ref{sec:morse}).
    \item We validate the learned score with a Gaussian-process audit in nine settings against two distinct references, contextual and training-set variation, stress-test its mechanism and population behavior on a pulse benchmark, and test downstream representation transfer on PTB-XL and IMPTC (Sections~\ref{sec:design}--\ref{sec:results}).
\end{itemize}

\section{Related Work}

DER trains the NIG head with the induced Student-$t$ negative log-likelihood and an evidence regularizer that scales total evidence by the prediction residual~\citep{amini2020deep}. \citet{meinert2023unreasonable} show that this coupling lets observation noise drive the reported uncertainty, and propose two remedies: normalizing the residual by the predictive Student-$t$ width (RN-DER, their Eq.~11) and a Gaussian mean-variance reparameterization that drops $\alpha$ (Gaussian-MV, their Eq.~12). Natural posterior networks instead tie the evidence count to a learned density~\citep{charpentier2022natural}. \citet{bengs2022pitfalls} expose a more general pitfall: minimizing a second-order loss does not incentivize faithful epistemic uncertainty. \citet{jurgens2024faithfully} confirm this empirically by comparing second-order predictions with reference distributions from repeated training-set sampling, and conclude that evidential scores support relative rather than absolute interpretation. \citet{jimenez2026position} argue that every current second-order estimator captures only a subset of the sources of epistemic uncertainty. These criteria characterize uncertainty induced by model or training-set variation; MORSE targets a complementary source, arising from changes in the visible evidence.

A second line of work supplies what the fitting objective cannot: observable variation in the predictions themselves. MC dropout and deep ensembles expose this variation through stochastic inference and independently trained models~\citep{gal2016dropout,lakshminarayanan2017simple}, and their disagreement is widely treated as an epistemic signal. Ensemble distribution distillation trains one model to reproduce the mean and spread of an ensemble~\citep{malinin2020ensemble}, and \citet{shen2024mirage} reinject model uncertainty into evidential learning by distilling a bootstrap ensemble; self-distribution distillation trains such a teacher and its student within one network~\citep{fathullah2022self}. What these methods resample is the model: parameters, initializations, or training sets. Because the observed input is held fixed, the resulting variation does not directly measure how much a prediction changes when the visible evidence itself changes, although it can correlate with that change (Section~\ref{sec:results}). Test-time augmentation perturbs the input instead and reads the resulting prediction variance as aleatoric uncertainty~\citep{ayhan2018tta,wang2019tta}.

Masked autoencoders learn representations by reconstructing hidden content from visible input~\citep{he2022mae}, adapted to time series through masked patch reconstruction~\citep{nie2023patchtst}, multiple masked views~\citep{dong2024simmtm}, and spatio-temporal masking for ECG~\citep{na2024stmem}; evidential objectives have also been combined with temporal imputation~\citep{ragab2025evidentially}. This setting permits a complementary intervention: the model is held fixed while the evidence available about a chosen target is varied, so predictive variation becomes attributable to observability rather than to the estimator. SimMTM aggregates several masked views to improve reconstruction; MORSE instead keeps the per-view predictions separate and learns their disagreement as a score. To our knowledge, no prior work uses the variation that masked pretraining already generates as target-specific supervision for the evidential allocation.

\section{Making the Allocation Observable}
\label{sec:theory}

\paragraph{Marginal allocation.}
For target $i$, a positive NIG-shaped head outputs $\mathbf m_i=(\gamma_i,\nu_i,\alpha_i,\beta_i)$ with $\nu_i>0$, $\alpha_i>1$, and $\beta_i>0$. Marginalizing the Gaussian mean and variance induces a Student-$t$ distribution with location $\gamma_i$, scale $s_i^2=\beta_i(1+\nu_i)/(\nu_i\alpha_i)$, and $d_i=2\alpha_i$ degrees of freedom (Figure~\ref{fig:motivation}A). The marginal likelihood therefore observes the four head outputs only through these three quantities. For $d_i>2$, its variance admits the structured allocation
\begin{align}
\Var(Y_i\mid\mathbf m_i)
&=\underbrace{\frac{\beta_i}{\alpha_i-1}}_{U_{\mathrm{data},i}}
+\underbrace{\frac{\beta_i}{\nu_i(\alpha_i-1)}}_{U_{\mathrm{mean},i}},
\label{eq:nig-decomposition}\\
U_{\mathrm{rel},i}
&:=\frac{U_{\mathrm{mean},i}}{U_{\mathrm{data},i}}=\frac{1}{\nu_i}.
\label{eq:urel-definition}
\end{align}
This allocation is not identified by the marginal. Fix any Student-$t$ marginal $(\gamma,s^2,d)$ with $d>2$. For any $\nu>0$, the choice $\alpha=d/2$ and $\beta=\alpha s^2\nu/(1+\nu)$ preserves the location, scale, and degrees of freedom, while Equation~\eqref{eq:urel-definition} gives $\urel=1/\nu$. As $\nu$ spans the positive reals, the same marginal admits every positive relative allocation (Proposition~\ref{prop:allocation}). An auxiliary regularizer may select a numerical parameterization, but it does not supply target-specific evidence for what the allocation means. MORSE retains $(\udata,\urel,\umean)$ only as structured coordinates and uses the Gaussian-form fitting objective defined below.

\paragraph{Observable bridge.}
We obtain such evidence by keeping a target hidden while varying its visible context. Fixing sequence $X$, target $t$, and predictor $\theta$, let $M\sim\pi_t$ denote an observation state conditioned on $t$ remaining hidden. Let the corresponding first-order predictive component have mean $\mu_\theta(X,M)$ and finite variance $v_\theta(X,M)$. The policy induces a mixture of these components, for which the law of total variance gives
\begin{equation}
\Var_{q_{\theta,\pi_t}}(Y_t\mid X)
=\E_{M\sim\pi_t}[v_\theta]
+\underbrace{\Var_{M\sim\pi_t}[\mu_\theta]}_{\robs(X,t;\pi_t)}.
\label{eq:mask-mixture-variance}
\end{equation}
Hence $\robs$ is the observable between-context variance of predictive means under the specified observation policy (Proposition~\ref{prop:between-context}). The decomposition has the within/between form of Rubin's combining rules for multiple imputation~\citep{rubin1987multiple}, and $\robs$ is the main-effect variance of the observation state in variance-based sensitivity analysis~\citep{sobol2001global}. It is relative to both that policy and the trained predictor. To construct a dimensionless target, MORSE uses the fitted local data scale as an operational proxy for $v_\theta$. With $\overline{U}_{\mathrm{data}}(X,t;\pi_t)=\E_{M\sim\pi_t}[\udata(M)]$,
\begin{equation}
\rnorm(X,t;\pi_t):=\frac{\robs(X,t;\pi_t)}{\overline{U}_{\mathrm{data}}(X,t;\pi_t)+\epsilon}.
\label{eq:bridge}
\end{equation}
\paragraph{Matching decompositions.}
With this choice, the two variance decompositions sit side by side:
\begin{equation}
\underbrace{\udata\,(1+\urel)}_{\text{NIG predictive variance}}
\qquad\text{and}\qquad
\underbrace{\overline{U}_{\mathrm{data}}\,(1+\rnorm)}_{\text{mask-mixture variance}}.
\label{eq:correspondence}
\end{equation}
The observability response $\rnorm$ occupies the position of $\urel=1/\nu$: both measure the variance of the mean in units of the data variance. Proposition~\ref{prop:allocation} shows that the marginal leaves exactly this coordinate free, and MORSE supplies its value by adopting the correspondence $\urel\leftrightarrow\rnorm$ as the definition of what $\urel$ reports. Then $\umean=\udata\urel$ tracks $\robs$, and the observation policy, rather than a Bayesian posterior, sets its meaning. The correspondence also gives $\nu$ a measured meaning. In the NIG hierarchy, $\nu$ counts the virtual observations that support the mean~\citep{amini2020deep}; matching $1/\nu$ to $\rnorm$ gives the effective count $\nu_{\mathrm{eff}}=\overline{U}_{\mathrm{data}}/\robs$ (Appendix~\ref{app:formal}). If the predictive mean varies across contexts with one quarter of the noise variance, the varying context is worth about four noisy observations of the target. Because $\rstar=0$ both when every admissible context pins down the target's mean and when no admissible context informs it (Proposition~\ref{prop:population}), $\nu_{\mathrm{eff}}$ counts only evidence that the policy varies; uncertainty about evidence withheld in every observation state is reported through $\udata$, not $\umean$.

\paragraph{Population limit.}
Let $\mu^\ast(X,M)$ be the population-optimal mean predictor for the same policy and write $\delta=\mu_\theta-\mu^\ast$. Then
\begin{equation}
\robs(\theta)=\rstar+\Var_M(\delta)+2\Cov_M(\mu^\ast,\delta),
\label{eq:population-decomposition}
\end{equation}
where $\rstar=\Var_M(\mu^\ast)$ (Proposition~\ref{prop:population}). As the predictor converges, $\robs(\theta)\to\rstar$, which is positive whenever changing the visible context changes the optimal mean. Variation of the prediction across training sets at a fixed mask involves only $\delta$, so it vanishes in the same limit up to optimization noise. When $\rstar>0$, contextual and training-set variation converge to different limits, so any agreement between them is a finite-sample coincidence rather than an identity; our evaluation compares the learned score with each of them separately.

\section{MORSE}
\label{sec:morse}

Figure~\ref{fig:method} operationalizes the correspondence above, and three design choices follow from it. The target is normalized by the data scale, so it is dimensionless. It supervises $\urel=1/\nu$ alone, so its gradients reach $\nu$ but not the data scale, which separates the two supervision paths at the output. And the fitting scale uses the bounded factor $1+\lambda_{\mathrm{cond}}h(\rnorm)$, the counterpart of $1+\urel$ in the NIG predictive variance.

\begin{figure*}[t]
    \centering
    \includegraphics[width=\textwidth]{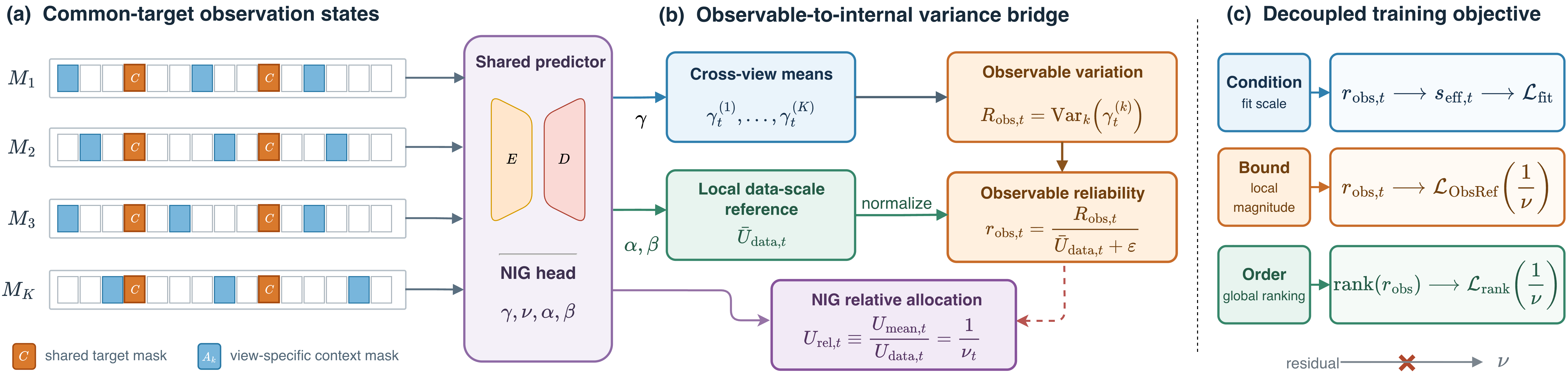}
    \caption{\textbf{MORSE overview.} (a) Common targets $C$ remain hidden in all $K$ views, while view-specific masks $A_k$ alter their supporting context. (b) Cross-context predictive-mean variance is normalized by the local data scale to form the observability response, which takes the position of $\urel=1/\nu$ in the NIG decomposition. (c) The detached response conditions fitting, bounds the local magnitude of the relative mean-uncertainty factor, and supervises its ordering. All masked positions train reconstruction, whereas only $C$ provides matched $K$-view supervision.}
    \label{fig:method}
\end{figure*}

\paragraph{Finite-view response.}
We sample a common target set $C$ and independently draw additional masks $A_k\subseteq C^c$ under a fixed masking policy, allowing the $A_k$ to overlap. For each common target, MORSE computes
\begin{align}
M_k&=C\cup A_k,\qquad k=1,\ldots,K,
\notag\\
\bar{\gamma}_t&=\frac{1}{K}\sum_{k=1}^{K}\gamma_{k,t},
\notag\\
\widehat R_{\mathrm{obs},t}&=\frac{1}{K}\sum_{k=1}^{K}\left(\gamma_{k,t}-\bar{\gamma}_t\right)^2,
\label{eq:obs-variance}\\
U_{\mathrm{data},t}^{\mathrm{ref}}&=\frac{1}{K}\sum_{k=1}^{K}U_{\mathrm{data},k,t},
\notag\\
r_{\mathrm{obs},t}&=\frac{\sg\!\left(\widehat R_{\mathrm{obs},t}\right)}{\sg\!\left(U_{\mathrm{data},t}^{\mathrm{ref}}\right)+\epsilon},
\qquad t\in C,
\label{eq:obs-response}
\end{align}
where $\gamma_{k,t}=\mu_\theta(X,M_k)_t$ is the NIG location predicted for target $t$ in view $k$ and $\sg$ denotes the stop-gradient operator. Only common targets define the detached response and the uncertainty losses, whereas every $t\in M_k$ contributes to reconstruction fitting. The view count $K$ and the masking policy are part of the operational definition of the target; full settings are given in Appendix~\ref{app:method}.

\paragraph{Conditioned fitting.}
We set $r_{\mathrm{obs},t}=0$ outside $C$ and use the bounded gate $h(r)=r/(1+r)$. The effective scale and fitting objective are
\begin{equation}
\begin{aligned}
s_{k,t}^{\mathrm{eff}}&=U_{\mathrm{data},k,t}\left[1+\lambda_{\mathrm{cond}}h(r_{\mathrm{obs},t})\right],\\
\mathcal L_{\mathrm{fit}}&=\frac{1}{K}\sum_{k=1}^{K}\operatorname*{Avg}_{t\in M_k}
\Bigl[\log\!\left(s_{k,t}^{\mathrm{eff}}+\epsilon\right)\\
&\qquad+\frac{(y_t-\gamma_{k,t})^2}{s_{k,t}^{\mathrm{eff}}+\epsilon}\Bigr].
\end{aligned}
\label{eq:fit}
\end{equation}
The gate limits the adjustment to $U_{\mathrm{data},k,t}\leq s_{k,t}^{\mathrm{eff}}<(1+\lambda_{\mathrm{cond}})U_{\mathrm{data},k,t}$. Because the response is detached, it conditions fitting without receiving gradients from $\mathcal L_{\mathrm{fit}}$. The effective scale is used only for optimization; reported moments use the base $U_{\mathrm{data},k,t}$ and $U_{\mathrm{mean},k,t}$. Thus $\mathcal L_{\mathrm{fit}}$ is Gaussian-form, contains no $\nu$, and does not invoke the Student-$t$ marginal.

\paragraph{Bound and order.}
Finite views provide a noisy local response rather than an exact label. We therefore define $T_t=\sg(r_{\mathrm{obs},t})+f_{\mathcal B}$, where $f_{\mathcal B}$ is a detached offset equal to the $q_f$-quantile of the response over valid common targets in the batch, lower-bounded by $\epsilon$. With $[x]_+=\max(x,0)$, magnitude supervision is
\begin{equation}
\ell_{\mathrm{ref}}(U,T)=[\log T-\log U]_+ +\omega[\log U-\log T-m]_+,
\label{eq:ref-loss}
\end{equation}
and averaging $\ell_{\mathrm{ref}}(U_{\mathrm{rel},k,t},T_t)$ over views and common targets defines the observable-reference loss $\mathcal L_{\mathrm{ref}}$ on the real-data tasks; the synthetic trainer applies $\ell_{\mathrm{ref}}$ once per common target to the view average $\bar U_{\mathrm{rel},t}=K^{-1}\sum_k U_{\mathrm{rel},k,t}$ (Appendix~\ref{app:method}). For ordering, let $z_t=K^{-1}\sum_k\log U_{\mathrm{rel},k,t}$ and let $H$ and $L$ index valid common targets in the upper and lower response quartiles. Then
\begin{align}
\mathcal L_{\mathrm{rank}}&=\E_{i\in H,\,j\in L}\left[\rho-(z_i-z_j)\right]_+,
\label{eq:rank-loss}\\
\mathcal L_{\method}&=\mathcal L_{\mathrm{fit}}+\lambda_{\mathrm{ref}}\mathcal L_{\mathrm{ref}}+\lambda_{\mathrm{rank}}\mathcal L_{\mathrm{rank}}.
\label{eq:morse-objective}
\end{align}
The envelope admits the zero-loss range $T_t\leq U_{\mathrm{rel},k,t}\leq e^mT_t$, constraining local magnitude, while $\mathcal L_{\mathrm{rank}}$ preserves the ordering induced by cross-context sensitivity. Because the envelope is a quantile-type loss, its population minimizer is a quantile-type functional of the finite-view response rather than its mean (Appendix~\ref{app:envelope}). Since $\mathcal L_{\mathrm{fit}}$ contains no $\nu$, these losses alone supervise $\urel$ at the output, and the residual has no direct path to it. This output-level decoupling does not preclude co-adaptation through the shared encoder.

\section{Evaluation Design}
\label{sec:design}

We ask three questions. \textbf{Semantic agreement:} does the learned score align with policy-relative contextual variation, and how does this differ from training-set variation? \textbf{Mechanism:} which MORSE components control magnitude, interval behavior, and ordering? \textbf{Transfer:} does MORSE pretraining improve low-label downstream performance?

\paragraph{Gaussian-process benchmark.}
The GP benchmark provides a contextual reference built from analytic Bayes conditional means. Length-128 windows are drawn from a zero-mean Mat\'ern-$5/2$ Gaussian process with unit signal variance and length scale .05, observed with noise whose standard deviation switches from .05 to .30 at a boundary sampled uniformly in $[.2,.8]$. Given the known kernel and noise variances, the Bayes conditional mean for any visible set is available in closed form~\citep{rasmussen2006gaussian}. Its variance over 128 fixed masks from the training policy defines the contextual reference, a finite-mask approximation of the population contextual variance; the learned models never receive the kernel, the noise variances, or the per-window noise boundary, so the reference is a known-noise Bayes target (Appendix~\ref{app:gp}). Each method is trained on five paired datasets of $N=2048$ windows and evaluated on a shared test bank of 1,024 windows. Alongside DER, RN-DER, Gaussian-MV, and MORSE, we evaluate five-member deep and bootstrap ensembles and 50-pass MC dropout with and without attention-probability dropout. All learned scores are averaged over the same four evaluation contexts. A separate bank of 20 first-order models trained on disjoint datasets defines the training-set reference $\vref$ (Appendix~\ref{app:gp}).

\paragraph{Pulse benchmark.}
The pulse benchmark serves as a mechanism and population-limit stress test. It extends the heteroscedastic pulse diagnostic of~\citet{meinert2023unreasonable}, which is pointwise and has fixed pulse and noise locations, to length-128 windows with independently sampled locations, so there is context to withhold. The observation-noise standard deviation changes from .01 to .10 at the sampled boundary, so noise level and reconstruction difficulty vary independently across the window. All second-order methods share the same two-layer Transformer, serialized datasets, normalization, optimization budget, and common-target masks, with $K=4$, a per-view mask rate of .50, a shared-target rate of .25, and ten paired seeds. For the training-size audit, we train MORSE at $N\in\{512,2048,8192,32768\}$ with ten paired anchors per size and compare the cross-context variance of its predictor with a large-sample reference from networks trained on unlimited freshly generated data (Appendix~\ref{app:pulse-audit}).

\paragraph{Real datasets.}
PTB-XL and IMPTC serve as downstream representation-transfer experiments. PTB-XL~\citep{wagner2020ptbxl} uses public 100-Hz, 12-lead ECGs and five diagnostic superclasses; we use the official patient-separated folds 1--8/9/10, pretrain only on the training folds, and fine-tune with 5\% or 10\% labeled training records. On IMPTC~\citep{hetzel2023imptc}, we construct a three-class future-event prediction task from 50-step observation and prediction windows with 46 motion and contextual features; the in-distribution setting uses 30\% labels and partial fine-tuning, and a separate evaluation holds out environment 5. Both use five seeds and validation-only checkpoint selection.

\paragraph{Metrics.}
Agreement with a reference variance is measured by log-variance MAE (scale), Spearman correlation (ordering only), and center RMSE (location of the predictive mean). The pulse mechanism study adds predictive fit, noise sensitivity ($L_{\mathrm{rel}}$ and $\Lmean$, absolute log-ratios of the high- to low-noise means of $\urel$ and $\umean$), error ranking, and the coverage gap and width of 90\% intervals (Appendix~\ref{app:synthetic}). Downstream evaluation uses task metrics, Brier score, ECE, and AURC~\citep{brier1950verification,guo2017calibration,geifman2017selective}. We report mean $\pm$ sample standard deviation over paired anchors or seeds; paired intervals are in the appendix.

\section{Results}
\label{sec:results}

\begin{table}[t]
\caption{\textbf{Agreement with two distinct GP uncertainty targets} ($N=2048$; five paired anchors; LogVar MAE, mean $\pm$ sample SD; lower is better). Contextual: variance of analytic Bayes conditional means over 128 fixed masks; Appendix~\ref{app:gp-scan} re-scores these checkpoints on an independent bank. Training set: variance across 20 independent first-order fits at fixed masks. All learned scores average the same four evaluation contexts; evidential heads, including MORSE, use one forward pass per context, MC dropout 50 draws, and each ensemble five members. Bold marks the best unrounded mean. The no-attention MC dropout row is retrained without attention-probability dropout (Appendix~\ref{app:mc-sites}).}
\label{tab:gp-bayes}
\centering
\footnotesize
\setlength{\tabcolsep}{3pt}
\begin{tabular}{lcc}
\toprule
& \multicolumn{2}{c}{LogVar MAE $\downarrow$}\\
\cmidrule(lr){2-3}
Method & Contextual & Training set \\
\midrule
DER & $2.150 \pm 0.105$ & $3.432 \pm 0.105$ \\
RN-DER & $2.414 \pm 0.071$ & $3.704 \pm 0.070$ \\
Gaussian-MV & $4.336 \pm 0.111$ & $5.632 \pm 0.111$ \\
Deep ensemble (5) & $2.166 \pm 0.140$ & $0.915 \pm 0.091$ \\
Bootstrap ensemble (5) & $1.607 \pm 0.092$ & $\mathbf{0.548 \pm 0.027}$ \\
MC dropout (50) & $0.807 \pm 0.013$ & $1.185 \pm 0.107$ \\
\quad no attention (50) & $1.012 \pm 0.065$ & $0.732 \pm 0.070$ \\
MORSE & $\mathbf{0.774 \pm 0.022}$ & $1.499 \pm 0.098$ \\
\bottomrule
\end{tabular}
\end{table}

\begin{figure*}[t]
    \centering
    \includegraphics[width=\textwidth]{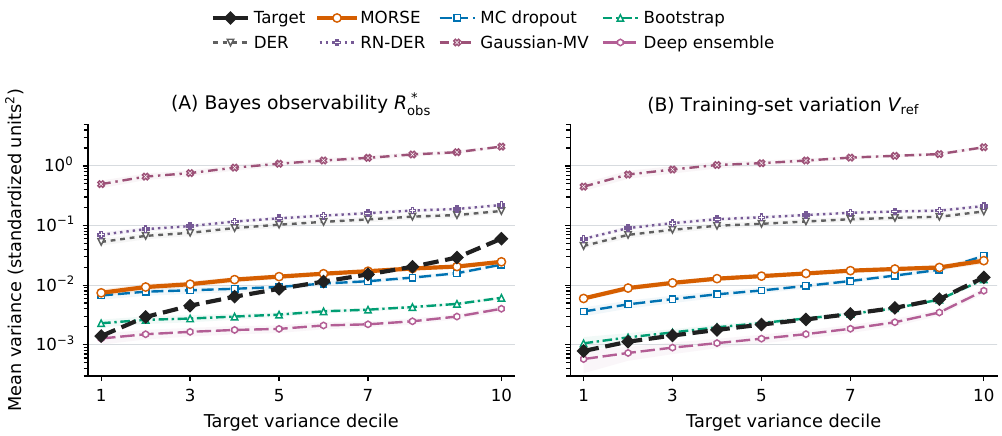}
    \caption{\textbf{Different estimators align with different variance targets.} Mean variance of each estimator within ten equal-count deciles of each target on the GP benchmark. \figthreeyaxis{} The deciles are formed separately in each panel from that panel's target, and all estimators share them; lines are means over five anchors and bands are anchor standard deviations. (A) Contextual target: variance of analytic Bayes conditional means over 128 fixed masks. (B) Training-set target $\vref$: variance across 20 independently trained first-order models at fixed masks.}
    \label{fig:gp-two-references}
\end{figure*}

\paragraph{Contextual alignment and target separation.}
Table~\ref{tab:gp-bayes} compares each learned score with both targets in the default setting. Against the contextual reference, MORSE has the lowest LogVar MAE (.774) and, in this setting, the highest Spearman correlation among the primary comparators on all five anchors; MC dropout follows at .807. Against training-set variation, the ranking reverses: the bootstrap ensemble is closest (.548), while MORSE is farther (1.499). Figure~\ref{fig:gp-two-references} shows that the reversal is a difference in level, so the two references are separable targets: an estimator can align closely with one and not the other. MC dropout lies between the references because its attention-probability dropout also removes context: retrained without it, its contextual error rises to 1.012 and its training-set error falls to .732, and the shift recurs in all nine settings (Appendix~\ref{app:mc-sites}). Across nine settings varying length scale and noise shape, MORSE keeps the lowest mean contextual LogVar MAE against both variants, and the estimator closest to $\vref$ always perturbs the model rather than the context (Appendix~\ref{app:gp-scan}). Against all-site MC dropout, the margin is large only at the shortest length scale, two paired intervals at length scale .10 include zero, and MC dropout has the higher Spearman correlation in five settings; the stable advantage is in scale, not ordering.

\paragraph{Score quality versus predictor quality.}
The GP result concerns the learned score, not the underlying predictor. All seven comparators have lower center RMSE than MORSE (.129--.139 versus .144), and RN-DER's raw 128-mask variance is closer to the contextual reference than MORSE's (.615 versus .673). MORSE's score also compresses the contextual target's range, about five-fold too high in the lowest decile and 2.4-fold too low in the highest. On the pulse proxy, five-member ensembles are closer than MORSE's score to both references, and MORSE's score is itself closer to $\vref$ than to the contextual proxy (LogVar MAE 1.545 versus 3.332; Table~\ref{tab:pulse-extended}), so the separation observed on GP does not carry over to that generator.

\paragraph{Amortization.}
Each four-context MORSE score in Table~\ref{tab:gp-bayes} costs four deterministic forward passes, against 200 for MC dropout and 20 for each five-member ensemble. At the same cost, the raw variance of a trained predictor over four masks has LogVar MAE near 1.2, against .774 for MORSE's score, and repeated masking needs 16 to 128 passes to overtake it (Appendix~\ref{app:gp}). This is an inference-time result; training uses four masked views per step.

\paragraph{Training-size audit.}
On the pulse benchmark, the cross-context variance of MORSE's predictor approaches the large-sample reference as data grow: its LogVar MAE falls from 1.987 at $N=512$ to .917 at $N=32768$ on all ten paired anchors (Table~\ref{tab:pulse-n}). The learned score follows the same trend at a larger error (4.477 to 1.687). Meanwhile the training-set reference contracts from .043 to .020, as Proposition~\ref{prop:population} predicts.
\paragraph{Checks of the correspondence.}
Equation~\eqref{eq:correspondence} suggests three checks. First, normalization removes units, not noise dependence; on the pulse benchmark, whose noise level changes tenfold, the noise sensitivity $L_{\mathrm{rel}}$ of the relative factor is .34 for MORSE against .53 to 1.49 for the evidential baselines (Table~\ref{tab:app-backbone-full}). Second, the normalized response is the right target: supervising $\urel$ with the raw cross-view variance instead raises GP scale error from .735 to 1.697 in a matched control (Appendix~\ref{app:components}). Third, Proposition~\ref{prop:envelope} predicts $\umean\approx 0.668\,\robs(\theta)$ under a Gaussian working assumption; the GP ratios (.50 to .85 on four anchors, .970 on the fifth) are near it, but the pulse ratios (3.6 to 4.4) are not, so the envelope does not fix the constant in general (Appendix~\ref{app:envelope}).

\paragraph{Mechanism.}
Table~\ref{tab:mechanism} (Appendix~\ref{app:synthetic}) isolates the roles of the components on the pulse benchmark. Among the model comparisons, RN-DER gives the lowest MSE and NLL and Gaussian-MV the strongest error ranking, while MORSE combines the lowest $\Lmean$, coverage gap, and interval width with near-best error ranking; a constant score can make $\Lmean$ look favorable, so we read it with coverage and width. Without ObsRef, interval width grows from 1.45 to 14.89; without Rank, $\Lmean$, coverage, and width worsen; restoring the residual path (MORSE-RC) slightly improves error ranking but degrades NLL, $\Lmean$, coverage, and width. A matched factorial over the components and heads on GP and pulse (Appendix~\ref{app:components}) assigns the scale gain to the supervised target: removing Bound raises GP scale error from .735 to 1.954 and restoring the residual path to 2.075, whereas Condition and Order mainly improve fitting and ranking.

\begin{table}[t]
\caption{PTB-XL low-label transfer at the 10\% label budget (five-seed mean $\pm$ sample standard deviation).}
\label{tab:ptb}
\centering
\small
\setlength{\tabcolsep}{4pt}
\begin{tabular}{lcccc}
\toprule
Method & AUROC $\uparrow$ & Macro-F1 $\uparrow$ & ECE $\downarrow$ & AURC $\downarrow$ \\
\midrule
Supervised & .827$\pm$.005 & .592$\pm$.011 & .069$\pm$.010 & .110$\pm$.011 \\
DAE & .826$\pm$.006 & .596$\pm$.012 & .079$\pm$.016 & .096$\pm$.008 \\
MAE & .842$\pm$.006 & .620$\pm$.005 & .043$\pm$.020 & .099$\pm$.008 \\
PatchTST & .823$\pm$.011 & .590$\pm$.015 & \textbf{.030}$\pm$.016 & .120$\pm$.014 \\
SimMTM & .849$\pm$.005 & .622$\pm$.008 & .043$\pm$.013 & .090$\pm$.007 \\
ST-MEM & .835$\pm$.010 & .604$\pm$.012 & .052$\pm$.020 & .106$\pm$.012 \\
\textbf{MORSE} & \textbf{.861}$\pm$.006 & \textbf{.638}$\pm$.007 & .041$\pm$.013 & \textbf{.085}$\pm$.003 \\
\bottomrule
\end{tabular}
\end{table}

\section{Discussion}
\label{sec:discussion}

In the GP experiments, observable supervision mainly improves variance magnitude. In the default setting, the evidential baselines have contextual Spearman correlations comparable to MORSE (.506--.526 versus .537), but substantially larger median variance ratios to the reference (8--79 versus 1.17; Tables~\ref{tab:gp-two-reference} and~\ref{tab:gp-context-intervals}). The matched ablations show the largest scale improvement from Bound, while Condition and Order affect fitting and ranking (Appendix~\ref{app:components}). Explicit, fixed-$\alpha$, and direct-variance heads achieve similar GP scale alignment, showing that the target can be learned with different parameterizations.

The GP results distinguish contextual from training-set variation empirically. Across all nine settings, MORSE has the lowest mean contextual log-variance error among the primary comparators, while the estimator closest to the training-set reference perturbs the model. Retraining MC dropout without attention-probability dropout raises its contextual error to 1.012 and lowers its training-set error to .732. This shift recurs in all nine settings (Appendix~\ref{app:mc-sites}). On the pulse proxy, ensembles are closer to both references, so the contextual advantage observed on GP is benchmark-dependent.

Variance alignment and interval coverage remain separate outcomes. MORSE's intervals cover 82\% at a nominal 95\% on GP and 83\% in the pulse reference audit. The envelope and ranking losses do not enforce nominal coverage; in the GP audit, MORSE-IS approaches nominal coverage while overestimating variance (Appendices~\ref{app:envelope} and~\ref{app:gp}). The observation policy also determines the contexts over which variance is measured and affects both the scale and accuracy of the score (Appendix~\ref{app:gp}). A deployment-oriented policy should therefore reflect the patterns of missing evidence expected in use. Validation against reference contextual variances is restricted to the synthetic benchmarks; PTB-XL and IMPTC evaluate representation transfer.

\section{Conclusion}

An uncertainty estimate is interpretable only if something outside the fitting objective constrains what it means. The NIG marginal leaves the allocation between data-scale and predictive-mean variance unidentified; masked representation learning supplies an observable behavioral signal by holding a target hidden while varying its visible context. Normalized by the data scale, this signal takes the position of the NIG factor $1/\nu$, which turns the unidentified evidence count into a measured count of the evidence the policy varies. MORSE learns a score grounded in this response and produces it in one forward pass per observation state. Against a contextual reference built from analytic Bayes conditional means, MORSE has the lowest log-variance error among the compared estimators in nine GP settings, while an independent training-set reference separates contextual from resampling-based variation. Because the target is defined by a stated observation policy, it can be matched to the evidence a system may lose. A quantity tied to a stated intervention can be checked; one asserted only by a loss cannot.

\section*{Acknowledgments}
This research has been partly funded by the Hessian State Ministry for Higher Education, Research, Science and the Arts (HMWK) as part of the LOEWE program (Kapitel 15 02, F\"orderbuchungskreis 2995, F\"orderprodukt 11 -- Landes-Offensive zur Entwicklung Wissenschaftlich-\"okonomischer Exzellenz).

\subsection*{AI Use Statement}

In this work, we used generative AI tools to help develop the conceptual framing of the added audits, to propose and refine their hypotheses, to design and provide feedback on the research methodology and experiments, to formulate mathematical claims, to provide ingredients for their proofs and assist in writing the proofs of Propositions~\ref{prop:population} and~\ref{prop:envelope}, to implement methods and experiment code (including the code that generates the synthetic datasets), to interpret results, and to assist with translation between Chinese and English. Cleaning or reformatting datasets and qualitative or thematic data analysis are not applicable to this work. Additionally, we used generative AI tools to draft and edit parts of the paper, to prepare and format the manuscript, to create and modify figures and tables, and to suggest the structure and title of the paper. We have reviewed all AI-assisted work. AI-assisted code was reviewed by the authors and checked with unit tests and an independent validation script covering input checksums, training manifests, and paired-anchor integrity; all reported tables and figures were regenerated from saved per-run artifacts; and all AI-assisted mathematical arguments were checked line by line by the authors. We take responsibility for the final content of this work, including text, claims, or artifacts produced with the aid of generative AI.

\bibliography{references}

\clearpage
\appendix
\section*{Appendix}
Appendix~\ref{app:formal} gives the formal results and Appendix~\ref{app:method} the complete method. The evidence follows the three questions of Section~\ref{sec:design}. Semantic agreement: the GP audits, including the nine-setting scan and the MC dropout-site study (Appendix~\ref{app:gp}), the finite-training-set audit (Appendix~\ref{app:reference}), and the pulse proxy and training-size audit (Appendix~\ref{app:pulse-audit}). Mechanism: the pulse-window results and ablations and the matched component and head study (Appendix~\ref{app:synthetic}). Transfer: the real-data protocols (Appendices~\ref{app:ptb} and~\ref{app:imptc}). Appendix~\ref{app:repro} covers reproducibility, and its final table maps each claim to its evidence. Tables and figures are numbered continuously with the main paper.

\section{Formal Results and Claim Boundaries}
\label{app:formal}

\subsection{Marginal Allocation Is Not Identified}

For one target, let a Normal-Inverse-Gamma (NIG)-shaped head output
$\mathbf m=(\gamma,\nu,\alpha,\beta)$ with $\nu>0$, $\alpha>1$, and
$\beta>0$. Its Student-$t$ marginal has location $\gamma$, scale
\begin{equation}
s^2=\frac{\beta(1+\nu)}{\nu\alpha},
\label{eq:app-student-scale}
\end{equation}
and degrees of freedom $d=2\alpha$. For $d>2$, its predictive variance is
\begin{equation}
\upred=\frac{d}{d-2}s^2
=\underbrace{\frac{\beta}{\alpha-1}}_{\udata}
+\underbrace{\frac{\beta}{\nu(\alpha-1)}}_{\umean},
\qquad \frac{\umean}{\udata}=\frac1\nu.
\label{eq:app-nig-slots}
\end{equation}

\begin{proposition}[Continuum of marginally equivalent allocations]
Fix any Student-$t$ marginal $(\gamma,s^2,d)$ with $s^2>0$ and $d>2$.
For every desired relative allocation $r>0$, there exists a valid NIG tuple
with the same marginal and $\umean/\udata=r$.
\label{prop:allocation}
\end{proposition}

\begin{proof}
Choose
\begin{equation}
\alpha=\frac d2,\qquad \nu=\frac1r,\qquad
\beta=\frac{\alpha s^2}{1+r}.
\label{eq:app-family}
\end{equation}
All positivity conditions hold. Substitution into
Equation~\eqref{eq:app-student-scale} gives
\[
\frac{\beta(1+\nu)}{\nu\alpha}
=\frac{\alpha s^2}{1+r}
  \frac{1+1/r}{(1/r)\alpha}
=s^2,
\]
while $2\alpha=d$ and the location remains $\gamma$. Finally,
$\umean/\udata=1/\nu=r$. Because $r$ is arbitrary, a continuum of internal
allocations produces the identical marginal predictive density.
\end{proof}

Proposition~\ref{prop:allocation} concerns evidence supplied by the marginal
likelihood. A regularizer, architectural constraint, or optimization bias may
select one tuple, but that selection does not create target-specific observable
evidence that its allocation has the intended meaning. MORSE addresses this
semantic supervision gap.

\subsection{Policy-Relative Total-Variance Derivation}

Fix an input sequence $X$, a target index $t$, and the event that $t$ is
hidden. Let $M\sim\pi_t$ be an observation state sampled from a specified
masking policy conditional on this event. For every admissible $M$, let
$q_\theta(y_t\mid X,M)$ have finite first and second moments,
\begin{equation}
\begin{aligned}
\mu_\theta(X,M)&=\E[Y_t\mid X,M],\\
v_\theta(X,M)&=\Var(Y_t\mid X,M).
\end{aligned}
\end{equation}
The policy induces
\begin{equation}
q_{\theta,\pi_t}(y_t\mid X)
=\E_{M\sim\pi_t}[q_\theta(y_t\mid X,M)].
\end{equation}

\begin{proposition}[Between-context variance under a masking policy]
Under the setup above,
\begin{equation}
\begin{aligned}
\Var_{q_{\theta,\pi_t}}(Y_t\mid X)
&=\E_{M\sim\pi_t}[v_\theta(X,M)]\\
&\quad+\Var_{M\sim\pi_t}[\mu_\theta(X,M)].
\end{aligned}
\label{eq:app-total-variance}
\end{equation}
Thus
$\robs(X,t;\pi_t):=\Var_{M\sim\pi_t}[\mu_\theta(X,M)]$ is exactly the
between-context predictive-mean component of the mask-policy-induced mixture.
\label{prop:between-context}
\end{proposition}

\begin{proof}
Introduce the joint distribution that first samples $M\sim\pi_t$ and then
$Y_t\sim q_\theta(\cdot\mid X,M)$. Conditioning on fixed $X$ and the target
event, the law of total variance gives
\[
\begin{aligned}
\Var(Y_t\mid X)
&=\E_M[\Var(Y_t\mid X,M)]\\
&\quad+\Var_M(\E[Y_t\mid X,M]).
\end{aligned}
\]
Substitution of $v_\theta$ and $\mu_\theta$ proves
Equation~\eqref{eq:app-total-variance}.
\end{proof}

The equality in Proposition~\ref{prop:between-context} is exact for the
mask-induced mixture. The connection to the NIG slot is structural:
$\robs$ occupies a between-mean position for variation over observation states,
whereas $\umean$ is the predictive-mean slot in the NIG hierarchy. $\robs$ is
defined over observation states, whereas Bayesian parameter uncertainty and
cross-training-set variation are defined over parameters and training sets
(Proposition~\ref{prop:population}).

\subsection{Finite-View Estimation and Scale Normalization}

For conditionally independent masks $M_1,\ldots,M_K\sim\pi_t$, MORSE computes
\begin{equation}
\widehat\robs
=\frac1K\sum_{k=1}^{K}(\mu_k-\bar\mu)^2.
\end{equation}
With divisor $K$,
$\E[\widehat\robs]=(K-1)\robs/K$. The estimator is downward-biased and
consistent under the stated sampling assumption. The final method fixes $K=4$
and treats both $K$ and $\pi_t$ as part of the operational definition of the
target.

Because $\widehat\robs$ is variance-valued, MORSE divides it by the mean local
data scale,
\begin{equation}
\rnorm{}_{t}
=\frac{\sg(\widehat{\robs}_{t})}
{\sg(K^{-1}\sum_k\udata{}_{k,t})+\epsilon}.
\label{eq:app-rnorm}
\end{equation}
The response is dimensionless and structurally comparable to
$\umean/\udata=1/\nu$. The comparison has a direct reading in the NIG
hierarchy, where $\mu\mid\sigma^2\sim\mathcal N(\gamma,\sigma^2/\nu)$ and
$\sigma^2\sim\mathrm{InvGamma}(\alpha,\beta)$. Then
$\Var(\mu)=\E[\sigma^2]/\nu=\udata/\nu$, so $1/\nu$ is the variance of the
mean in units of the data variance, and $\nu$ is the number of noisy
observations whose average would carry this much uncertainty about the mean.
The normalized response measures the same ratio for the variation of the mean
across observation states, which gives the effective count
$\nu_{\mathrm{eff}}=\overline{U}_{\mathrm{data}}/\robs$.
Because $\rstar=0$ whenever the optimal mean does not change across admissible
contexts, including when no admissible context informs the target,
$\nu_{\mathrm{eff}}$ counts only the evidence that the policy varies.
Proposition~\ref{prop:envelope}, under a Gaussian working assumption, gives the
constant with which the envelope alone matches the two
(Appendix~\ref{app:envelope}). Detaching numerator and denominator prevents the
reference loss from reducing itself by collapsing cross-view means or
inflating its own denominator. Repeated optimization can still produce slower
co-adaptation through the shared network, so detachment is not an independence
guarantee.

\subsection{Direct-Path Decoupling}

The final fitting surrogate is
\begin{equation}
\mathcal L_{\mathrm{fit}}
=\log(s_{\mathrm{eff}}+\epsilon)
+\frac{(y-\mu)^2}{s_{\mathrm{eff}}+\epsilon},
\end{equation}
where $s_{\mathrm{eff}}$ depends on $\udata$ and detached $\rnorm$, but not on
$\nu$. Hence
$\partial\mathcal L_{\mathrm{fit}}/\partial\nu=0$ at the output level. The
relative factor is directly trained by ObsRef and Rank only. Shared encoder
parameters still receive the sum of all representation gradients; this is
direct-path decoupling, not statistical independence of output variables.

\subsection{Population Limit of the Observability Response}
\label{app:population}

Fix $X$, a hidden target $t$, and a policy $\pi_t$. Let $\mu^\ast(X,M)$ be a population-optimal mean predictor for the observation state $M$ (for squared loss, the conditional mean of $Y_t$ given the visible context under $M$), and let $\mu_\theta(X,M)$ be a trained predictor. Assume both are square-integrable over $M\sim\pi_t$ and write $\delta_\theta=\mu_\theta-\mu^\ast$.

\begin{proposition}[Decomposition and population limit]
\label{prop:population}
Under the assumptions above,
\begin{equation}
\robs(\theta)=\Var_{M}[\mu_\theta]
=\rstar+\Var_{M}[\delta_\theta]+2\Cov_{M}(\mu^\ast,\delta_\theta),
\qquad \rstar:=\Var_{M}[\mu^\ast].
\label{eq:app-population}
\end{equation}
If $\E_{M}[\delta_{\theta_n}^2]\to0$ for a sequence of predictors $\theta_n$, then $\robs(\theta_n)\to\rstar$. Moreover, $\rstar=0$ if and only if $\mu^\ast(X,M)$ is $\pi_t$-almost surely constant in $M$, that is, if changing the visible context under the policy never changes the conditional mean of the target. For a fixed state $M_0$, the variance of $\mu_{\theta_D}(X,M_0)$ over training sets $D$ equals $\Var_D[\delta_{\theta_D}(X,M_0)]$ and tends to zero whenever $\E_D[\delta_{\theta_D}(X,M_0)^2]\to0$.
\end{proposition}

\begin{proof}
The first identity is the variance of a sum, $\Var(a+b)=\Var(a)+\Var(b)+2\Cov(a,b)$, with $a=\mu^\ast$ and $b=\delta_\theta$. By Cauchy--Schwarz, $|\Cov_M(\mu^\ast,\delta_{\theta_n})|\leq(\rstar)^{1/2}\Var_M[\delta_{\theta_n}]^{1/2}$, and $\Var_M[\delta_{\theta_n}]\leq\E_M[\delta_{\theta_n}^2]\to0$, so both correction terms vanish. A variance is zero exactly when its argument is almost surely constant. At fixed $M_0$, $\mu^\ast(X,M_0)$ does not depend on $D$, so the variance over $D$ is that of $\delta_{\theta_D}(X,M_0)$, which is bounded by its second moment.
\end{proof}

Proposition~\ref{prop:population} explains why agreement with a training-set reference can only be a finite-sample phenomenon. In practice, the variance over training sets does not reach zero because initialization and optimization add a floor; the pulse proxies in Appendix~\ref{app:pulse-audit} give an empirical estimate of that floor. The decomposition also shows that $\robs(\theta)$ can approach $\rstar$ from below when a finite-sample predictor shrinks toward context-averaged predictions, which makes $\Cov_M(\mu^\ast,\delta_\theta)$ negative.

\subsection{What the Envelope Targets}
\label{app:envelope}

Write $z=\log T$ for the finite-view target and $u=\log U$ for the head output. For fixed $u$, the envelope loss of Equation~\eqref{eq:ref-loss} is $(z-u)_+ +\omega(u-z-m)_+$.

\begin{proposition}[Population minimizer of the envelope]
\label{prop:envelope}
Let $z$ have a continuous distribution with finite first moment, and let $\omega>0$ and $m\geq0$. Every minimizer $u^\ast$ of $g(u)=\E[(z-u)_+ +\omega(u-z-m)_+]$ satisfies $P(z>u^\ast)=\omega\,P(z<u^\ast-m)$. For $\omega=1$ and a distribution of $z$ symmetric about its median $z_{0.5}$, $u^\ast=z_{0.5}+m/2$.
\end{proposition}

\begin{proof}
$g$ is convex, and its derivative is $g'(u)=-P(z>u)+\omega P(z<u-m)$, which is continuous for continuous $z$. Setting $g'(u^\ast)=0$ gives the first claim. Under symmetry, $P(z>z_{0.5}+m/2)=P(z<z_{0.5}-m/2)$, so $u^\ast=z_{0.5}+m/2$ solves the equation.
\end{proof}

The envelope therefore targets a quantile-type functional of the finite-view response, not its mean. Under a Gaussian working assumption for the $K=4$ view predictions, $K\widehat\robs/\robs\sim\chi^2_{K-1}$, and Monte Carlo simulation with a fixed data scale and no batch offset gives $\exp(u^\ast)\approx0.668\,\robs/\overline U_{\mathrm{data}}$, that is $\umean\approx0.668\,\robs$, for the synthetic setting $(m,\omega)=(.25,1)$, and $\umean\approx1.195\,\robs$ for the real-data setting $(m,\omega)=(1,.5)$. The calculation treats the envelope in isolation, without the fitting and ranking terms. On the GP benchmark, the median ratio of MORSE's $\umean$ to its own 128-mask response lies in $[.50,.85]$ for four of five anchors and equals .970 for the fifth, in line with this prediction. On the pulse benchmark the corresponding ratios are between 3.6 and 4.4, so the prediction does not transfer to that generator.

A mean-targeting alternative replaces the envelope with the Itakura--Saito loss $\exp(\tau-u)-(\tau-u)-1$ on the bias-corrected target $\tau=\log(K\,\widehat\robs/((K-1)(\overline U_{\mathrm{data}}+\epsilon))+\epsilon)$, applied separately in each view. Its population minimizer satisfies $\exp(u^\ast)=\E[\exp\tau]$, the conditional mean of the bias-corrected ratio. With $\overline U_{\mathrm{data}}$ held fixed and independent views, this ratio is unbiased for $\robs/(\overline U_{\mathrm{data}}+\epsilon)$ up to the added $\epsilon$; because the denominator is estimated from the same views, it is not exactly unbiased in general. In practice this variant (MORSE-IS) overestimated the model's own response on every anchor (median ratios 1.66--2.08 on GP and 5.0--7.2 on pulse) and was less accurate than the envelope against the GP Bayes reference (LogVar MAE 1.109 versus .774; paired difference $+.335$, 95\% interval $[.301,.375]$, envelope better on 5/5 anchors). We therefore retain the envelope as the method and report MORSE-IS only as a control; we have not isolated whether finite-view teacher noise, the moving self-teacher, or the interaction with conditioned fitting causes the discrepancy.

\section{Complete Method Specification}
\label{app:method}

\subsection{Observation-State Construction}

For each input, the sampler first draws a common target set $C$ and then draws
independent additional masks $A_k$ for $k=1,\ldots,K$. View $k$ masks
$M_k=C\cup A_k$. Every position in $C$ is therefore predicted in every view,
while $A_k$ changes the visible supporting context. The final policy uses four
views, 50\% total masking per view, and 25\% shared targets.

All masked positions in every $M_k$ contribute to reconstruction fitting, but
only positions in $C$ have enough matched predictions to define
$\widehat{\robs}_{t}$ and $\rnorm{}_{t}$. No single view is privileged: all $K$ NIG
outputs participate symmetrically in fitting, the reference denominator, and
supervision of the relative mean-uncertainty factor. This common-target construction is what makes
cross-view variance interpretable as variation for the same hidden quantity.

\subsection{Fitting, Envelope, and Rank Objectives}

With $h(r)=r/(1+r)$, the effective scale and fitting objective are
\begin{align}
s_{k,t}^{\mathrm{eff}}
&=\udata{}_{k,t}\left[1+\lambda_{\mathrm{cond}}h(\rnorm{}_{t})\right],\\
\mathcal L_{\mathrm{fit}}
&=\frac1K\sum_{k=1}^K\operatorname*{Avg}_{t\in M_k}
\left[
\log(s_{k,t}^{\mathrm{eff}}+\epsilon)
+\frac{(y_t-\mu_{k,t})^2}{s_{k,t}^{\mathrm{eff}}+\epsilon}
\right].
\end{align}
The bounded transform prevents a large finite-view response from making the
conditional scale arbitrarily large.

Let $T_t=\rnorm{}_{t}+f$, where $f$ is a robust nonnegative floor. MORSE uses
\begin{equation}
\ell_{\mathrm{ref}}(U,T)
=[\log T-\log U]_+
+\omega[\log U-\log T-m]_+.
\end{equation}
On the real-data tasks, $\mathcal L_{\mathrm{ref}}$ averages this loss for
$U=\urel{}_{k,t}=1/\nu_{k,t}$ across views and common targets; the synthetic
trainer applies it once per common target to the view average
$U=K^{-1}\sum_k\urel{}_{k,t}$ (Table~\ref{tab:app-hparams}). The first hinge
discourages an under-responsive factor; the weighted second hinge permits a
controlled upper margin rather than enforcing pointwise equality to a noisy
finite-$K$ target. We call this the observable-reference (ObsRef) loss.

Let $H$ and $L$ be the upper and lower quartiles of detached $\rnorm$ in a
batch, and let $z=\log\urel=-\log\nu$. The ordinal term and final objective are
\begin{align}
\mathcal L_{\mathrm{rank}}
&=\E_{i\in H,j\in L}[\rho-(z_i-z_j)]_+,\\
\mathcal L_{\method}
&=\mathcal L_{\mathrm{fit}}
+\lambda_{\mathrm{ref}}\mathcal L_{\mathrm{ref}}
+\lambda_{\mathrm{rank}}\mathcal L_{\mathrm{rank}}.
\end{align}
The envelope provides local magnitude control; Rank prevents a solution that
fits a common magnitude without preserving between-target ordering.

\subsection{Frozen Main Configuration}

Table~\ref{tab:app-hparams} records the objective settings recovered from the
executed configurations. The target floor is zero in the strict synthetic
trainer and the 0.10 quantile on the two real-data tasks. In strict-synthetic training the rank weight is ramped linearly from the first
epoch and reaches its full value after 20\% of training; the task-specific
PTB-XL and IMPTC trainers use no ramp, so the rank term is active with
$\lambda_{\mathrm{rank}}=.05$ from the first epoch.
\begin{table*}[!t]
\caption{Confirmed frozen MORSE settings. All reported main-method runs use
the decoupled objective, so $\lambda_\nu=0$.}
\label{tab:app-hparams}
\centering
\small
\begin{tabular}{lccc}
\toprule
Setting & Synthetic & PTB-XL & IMPTC \\
\midrule
Views $K$ & 4 & 4 & 4 \\
Mask rate per view / shared-target rate & .50 / .25 & .50 / .25 & .50 / .25 \\
$\lambda_{\mathrm{cond}}$ / $\lambda_\nu$ & 1 / 0 & 1 / 0 & 1 / 0 \\
$\lambda_{\mathrm{ref}}$ / $\lambda_{\mathrm{rank}}$ & .01 / .05 & .05 / .05 & .05 / .05 \\
ObsRef margin $m$ / upper weight $\omega$ & .25 / 1 & 1 / .5 & 1 / .5 \\
Target-floor quantile $f$ & 0 & .10 & .10 \\
ObsRef aggregation over views & view average & per view & per view \\
Rank margin $\rho$; high/low quantiles & .10; .25/.25 & .10; .25/.25 & .10; .25/.25 \\
Rank-weight ramp & linear, first 20\% & none & none \\
Positive-parameter / arithmetic $\epsilon$ & $10^{-4}$ / $10^{-6}$ & same & same \\
\bottomrule
\end{tabular}
\end{table*}

\section{Synthetic Protocols and Expanded Results}
\label{app:synthetic}

\subsection{Diagnostic Hierarchy}

The synthetic evaluation has three levels. Fixed one-dimensional pulses verify
the Deep Evidential Regression (DER) and Meinert loss implementations under
familiar signal/noise regions.
Parameterized pulses vary pulse and noise location across tasks. The main
length-128 windowed benchmark samples pulse and noise boundaries independently
and evaluates masked reconstruction with a latent anchor. The first two levels
are diagnostic reproductions; the windowed task supplies the controlled
mechanism and reference-distribution evidence used in the main claims.
For direct diagnostics, $x\in[0,1]$ and the fixed pulse has center .5,
half-width .025, unit amplitude, and noise standard deviations .01/.10 on
the two sides of a boundary at .5. Parameterized tasks independently sample
the pulse center from $\mathcal U(.25,.75)$, the noise boundary from
$\mathcal U(.20,.80)$, a background from $\mathcal U(-.05,.05)$, a
half-width multiplier from $\mathcal U(.55,1.45)$, and an amplitude from
$\mathcal U(.65,1.45)$. Independence is essential: the model cannot infer the
noise boundary by memorizing the pulse location.

\subsection{Direct One-Dimensional Diagnostics}

Following the main paper, we call Meinert et al.'s residual-normalized DER
objective RN-DER (their Eq.~11) and their Gaussian mean--variance fitting
objective Gaussian-MV (their Eq.~12)~\citep{meinert2023unreasonable}.
The direct models use a 1--64--64--64--4 rectified-linear-unit multilayer
perceptron (ReLU MLP), Adam with learning rate
$10^{-3}$, batch size 100, 1,000 training samples, and 600 epochs. DER uses
its residual-evidence coefficient .01; RN-DER uses normalized-residual
coefficient .01 with $p=2$; Gaussian-MV uses $\lambda_\nu=.01$. The fixed
diagnostic contains 50 independent fits of one task. The parameterized
diagnostic contains 50 independently sampled tasks.

For high-noise region $H$ and low-noise region $L$, the direct audit reports
\begin{align}
L_{1/\sqrt{\nu}}
&=\left|\log\frac{\E_H[1/\sqrt{\nu}]}{\E_L[1/\sqrt{\nu}]}\right|,\\
L_{\mathrm{rel}}
&=\left|\log\frac{\E_H[1/\nu]}{\E_L[1/\nu]}\right|,\\
L_{\mathrm{mean}}
&=\left|\log\frac{\E_H[\umean]}{\E_L[\umean]}\right|.
\end{align}
The first matches the orange standard-deviation-level quantity in the
Meinert-style visualization, the second audits the relative factor used by
MORSE, and the third audits absolute NIG mean variance. None is sufficient
alone: constant uncertainty gives zero. Error AUROC therefore asks whether
$\umean$ identifies the top 20\% squared-error tokens; selective excess is the
area between its risk--coverage curve and oracle error ordering; coverage gap
is the absolute error from nominal 90\% coverage; and Width is the average
90\% interval width. Here AUROC denotes area under the receiver operating
characteristic curve. Tables use mean squared error (MSE) and negative
log-likelihood (NLL) for predictive fitting.

\begin{figure*}[!t]
  \centering
  \includegraphics[width=.98\textwidth]{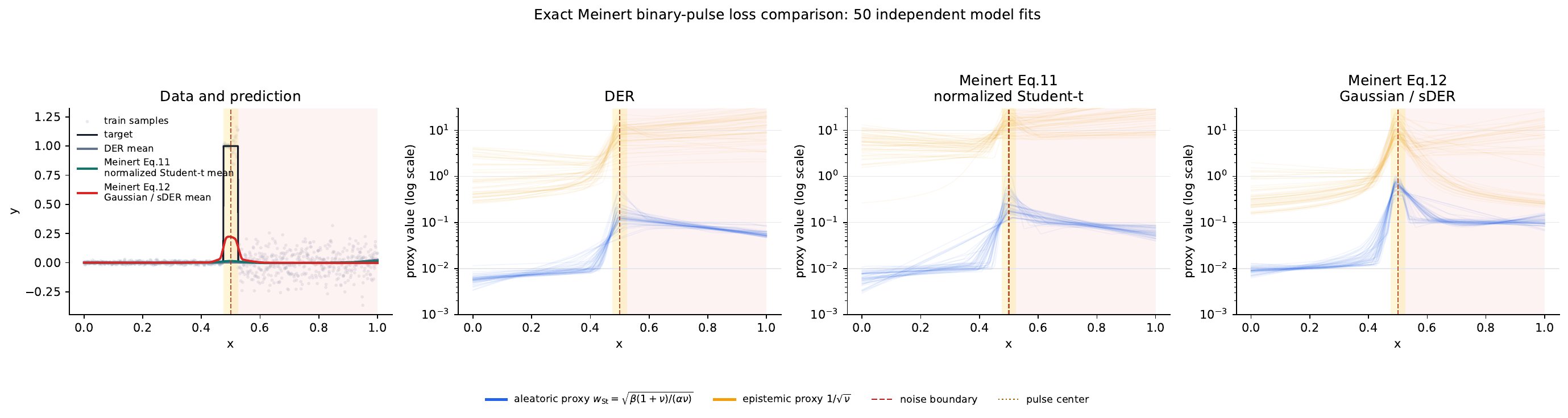}
  \caption{\textbf{Exact fixed-pulse diagnostic.} From left to right, the
  panels show the data-generating process with predictive means and 50
  independently fitted DER, RN-DER, and Gaussian-MV models.
  Blue is the predictive
  data-width proxy
  $\sqrt{\beta(1+\nu)/(\alpha\nu)}$ and orange is the relative
  mean-uncertainty width $1/\sqrt{\nu}$, matching the original diagnostic's
  plotting convention; its legend names them aleatoric and epistemic proxies. This is an implementation check; the scale-sensitive
  metrics in Table~\ref{tab:app-direct-fixed} prevent a visually flat response
  from being treated as sufficient evidence of a correct allocation. The supplied diagnostic's
  compact Eq.11/Eq.12 labels map to RN-DER/Gaussian-MV, respectively.}
  \label{fig:app-fixed}
\end{figure*}

\begin{table*}[!t]
\caption{Fixed-pulse direct diagnostic, 50 independent fits (mean $\pm$
sample standard deviation). Leakage metrics are accompanied by external utility and
interval diagnostics because a constant score can obtain low leakage.}
\label{tab:app-direct-fixed}
\centering
\small
\setlength{\tabcolsep}{4pt}
\begin{tabular}{lrrrrr}
\toprule
Method & MSE $\downarrow$ & NLL $\downarrow$ &
$L_{1/\sqrt{\nu}}\downarrow$ & $L_{\mathrm{rel}}\downarrow$ &
$L_{\mathrm{mean}}\downarrow$\\
\midrule
DER & .0498$\pm$.0008 & \textbf{$-2.0920$}$\pm$.1483 & 2.2526$\pm$1.1065 & 4.4823$\pm$2.2005 & 6.1481$\pm$2.0911\\
RN-DER & .0488$\pm$.0020 & $-2.0164\pm$.1641 & \textbf{1.3639}$\pm$.9395 & \textbf{2.7457}$\pm$1.8722 & \textbf{5.0433}$\pm$1.7440\\
Gaussian-MV & \textbf{.0362}$\pm$.0163 & $-1.9573\pm$.1258 & 1.4991$\pm$1.0410 & 3.0724$\pm$2.0066 & 7.4254$\pm$2.1064\\
\bottomrule
\end{tabular}

\medskip
\begin{tabular}{lrrrr}
\toprule
Method & Error AUROC $\uparrow$ & Sel. excess $\downarrow$ &
Cov. gap $\downarrow$ & Width $\downarrow$\\
\midrule
DER & .7563$\pm$.1594 & .0235$\pm$.0136 & .0519$\pm$.0188 & \textbf{.1753}$\pm$.0280\\
RN-DER & .7339$\pm$.2127 & .0085$\pm$.0116 & \textbf{.0497}$\pm$.0188 & .2188$\pm$.0276\\
Gaussian-MV & \textbf{.8552}$\pm$.1501 & \textbf{.0002}$\pm$.0002 & .0758$\pm$.0109 & .3263$\pm$.0484\\
\bottomrule
\end{tabular}
\end{table*}

\begin{figure*}[!t]
  \centering
  \includegraphics[width=.47\textwidth]{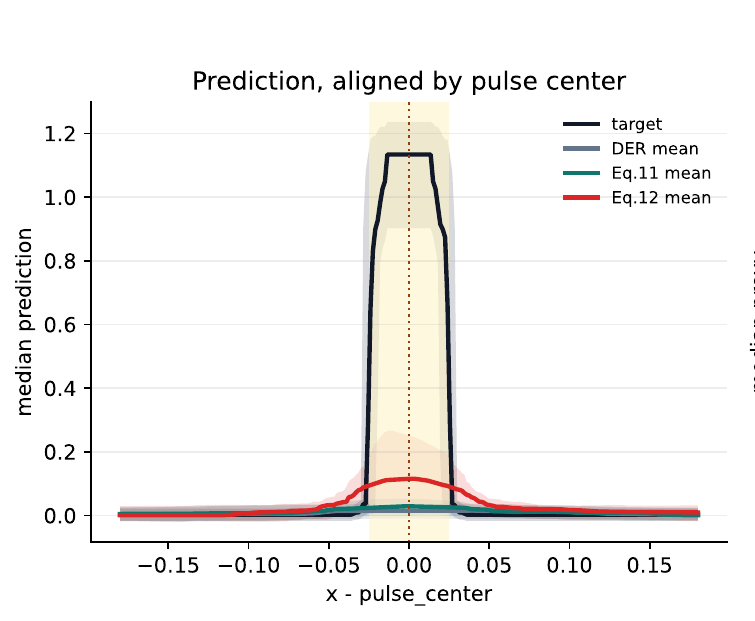}\hfill
  \includegraphics[width=.47\textwidth]{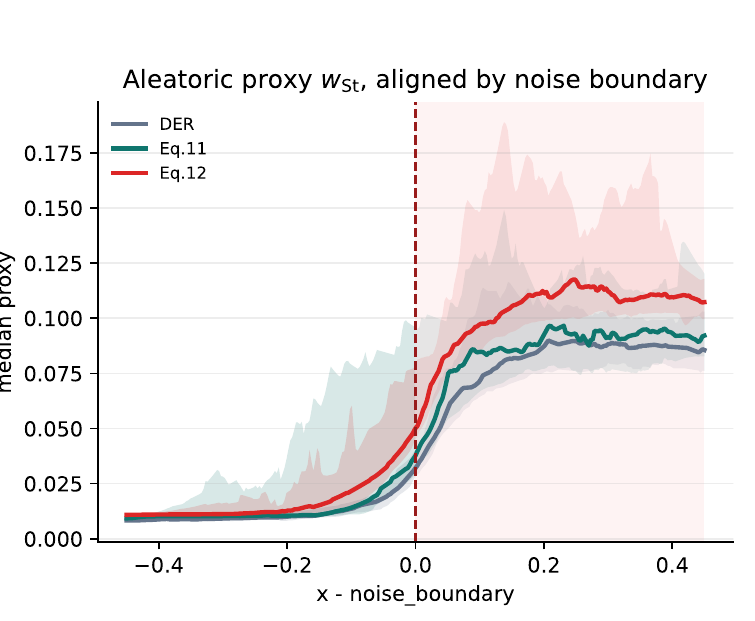}
  \par
  \includegraphics[width=.47\textwidth]{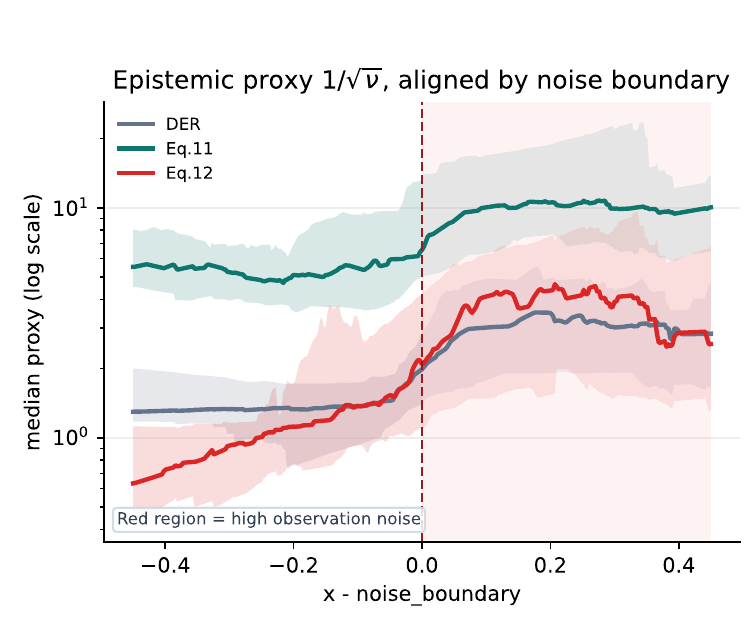}
  \caption{\textbf{Parameterized direct diagnostic.} From upper left to lower
  center: prediction aligned by latent pulse center, data-scale response
  aligned by the independently sampled noise boundary, and relative
  mean-uncertainty response aligned by that boundary. Curves show the median
  and bands the interquartile range across 50 tasks. Alignment is used only for
  visualization, not for training or metric computation. Panel titles keep the
  original diagnostic's aleatoric and epistemic proxy names. The supplied
  diagnostic's compact Eq.11/Eq.12 labels map to RN-DER/Gaussian-MV,
  respectively.}
  \label{fig:app-parameterized}
\end{figure*}

\begin{table*}[!t]
\caption{Parameterized direct diagnostic over 50 independently sampled tasks
(mean $\pm$ sample standard deviation).}
\label{tab:app-direct-parameterized}
\centering
\small
\setlength{\tabcolsep}{4pt}
\begin{tabular}{lrrrrr}
\toprule
Method & MSE $\downarrow$ & NLL $\downarrow$ &
$L_{1/\sqrt{\nu}}\downarrow$ & $L_{\mathrm{rel}}\downarrow$ &
$L_{\mathrm{mean}}\downarrow$\\
\midrule
DER & .0591$\pm$.0282 & \textbf{$-2.0797$}$\pm$.3317 & \textbf{.9234}$\pm$.7862 & \textbf{1.8467}$\pm$1.5562 & \textbf{3.2412}$\pm$1.7635\\
RN-DER & .0573$\pm$.0283 & $-1.9756\pm$.3310 & 1.0276$\pm$.8912 & 2.0126$\pm$1.7041 & 3.9321$\pm$1.7180\\
Gaussian-MV & \textbf{.0422}$\pm$.0283 & $-1.9377\pm$.3587 & 1.2887$\pm$1.0253 & 2.5494$\pm$1.9830 & 5.5316$\pm$2.7207\\
\bottomrule
\end{tabular}

\medskip
\begin{tabular}{lrrrr}
\toprule
Method & Error AUROC $\uparrow$ & Sel. excess $\downarrow$ &
Cov. gap $\downarrow$ & Width $\downarrow$\\
\midrule
DER & .7675$\pm$.2003 & .0254$\pm$.0352 & .0528$\pm$.0149 & \textbf{.1837}$\pm$.0898\\
RN-DER & .7605$\pm$.2282 & .0064$\pm$.0139 & \textbf{.0527}$\pm$.0190 & .2304$\pm$.0702\\
Gaussian-MV & \textbf{.8749}$\pm$.1297 & \textbf{.0005}$\pm$.0021 & .0774$\pm$.0130 & .3360$\pm$.0766\\
\bottomrule
\end{tabular}
\end{table*}

\subsection{Parameterized Pulse-Window Protocol}

Each sequence has 128 tokens, with 2,048/1,024 train/test windows generated
from the parameterized process above. Targets are standardized using training
statistics. All masked methods use the same two-layer Transformer ($d=64$,
four heads), AdamW with learning rate $10^{-3}$ and weight decay $10^{-4}$,
batch size 64, and 100 epochs.

\begin{figure*}[!t]
  \centering
  \includegraphics[width=.94\textwidth,trim=0 0 0 27bp,clip]{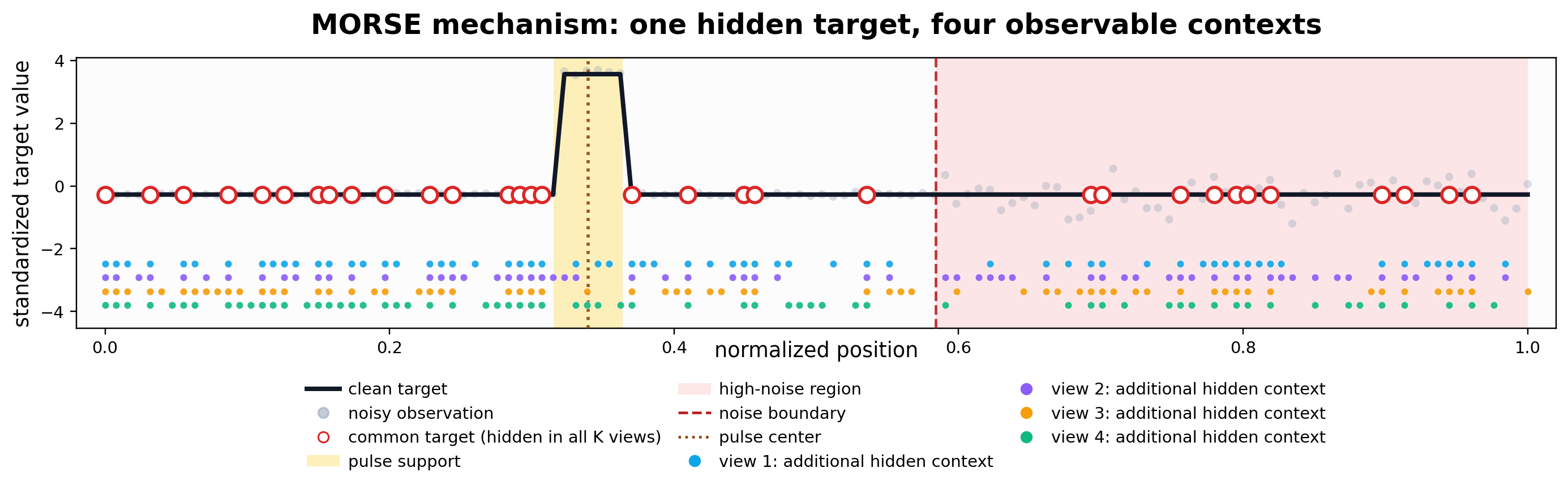}
  \caption{\textbf{Common-target observation policy.} One sequence is shown,
  with the legend below. Red circles mark the common targets $C$, hidden in all
  four views; the four rows of dots beneath the signal mark each view's
  additional hidden context $A_k$. The pulse and noise boundary vary
  independently. Every cross-view prediction therefore refers to the same
  target while its supporting evidence changes.}
  \label{fig:app-observation-policy}
\end{figure*}

\subsection{Full Ten-Seed Backbone Comparison}

$L_{\mathrm{rel}}$ measures undesired sensitivity of the relative factor to
the heteroscedastic noise mechanism. $L_{\mathrm{mean}}$ is the absolute
log-ratio between average $\umean$ in high- and low-noise regions after pooling
tokens within each seed. A constant score can make the latter small, so it must
be interpreted with coverage, interval width, error AUROC, and selective
excess.

\begin{table*}[!t]
\caption{Parameterized pulse-window backbone comparison over ten paired seeds
(mean $\pm$ sample standard deviation).}
\label{tab:app-backbone-full}
\centering
\setlength{\tabcolsep}{3.0pt}
\small
\resizebox{\textwidth}{!}{%
\begin{tabular}{lccccccc}
\toprule
Method & MSE $\downarrow$ & NLL $\downarrow$ & $L_{\mathrm{rel}}$ $\downarrow$ & $L_{\mathrm{mean}}$ $\downarrow$ & Err. AUROC $\uparrow$ & Sel. excess $\downarrow$ & Cov. gap $\downarrow$ \\
\midrule
Masked DER
& .6188$\pm$.1811 & $-.2508\pm$.0516 & 1.4947$\pm$.2270 & .5562$\pm$.4302 & .8271$\pm$.0228 & .0895$\pm$.0470 & .0460$\pm$.0166 \\
RN-DER
& \textbf{.5151}$\pm$.0588 & \textbf{$-.2565$}$\pm$.0268 & 1.1061$\pm$.2679 & .3149$\pm$.4609 & .8319$\pm$.0087 & .0360$\pm$.0040 & .0512$\pm$.0085 \\
Gaussian-MV
& .5416$\pm$.0423 & $-.0996\pm$.0499 & .5343$\pm$.2257 & .5772$\pm$.1008 & \textbf{.8785}$\pm$.0168 & \textbf{.0289}$\pm$.0025 & .0934$\pm$.0020 \\
\textbf{MORSE}
& .5331$\pm$.0371 & $-.1681\pm$.0298 & \textbf{.3424}$\pm$.0906 & \textbf{.1253}$\pm$.1092 & .8705$\pm$.0070 & .0295$\pm$.0034 & \textbf{.0176}$\pm$.0086 \\
\bottomrule
\end{tabular}}
\end{table*}

RN-DER remains strongest for reconstruction MSE and evaluation NLL, while
Gaussian-MV gives the highest error AUROC and smallest selective excess. MORSE gives
the smallest leakage diagnostics and coverage gap while remaining close on
selective utility. These complementary outcomes motivate treating MORSE as a
balanced method rather than counting metric wins.

\subsection{Full Ten-Seed Core Ablations}

\begin{table*}[!t]
  \caption{Parameterized pulse-window results (ten-seed mean $\pm$ sample standard deviation). MORSE-RC restores the direct residual-to-$\nu$ gradient. $\Lmean$ measures pooled high/low-noise sensitivity, and Width$^\dagger$ is the mean 90\% moment-interval width, to be read jointly with CovGap. Bold marks the best mean within each block.}
  \label{tab:mechanism}
  \centering
  \small
  \setlength{\tabcolsep}{4pt}
  \resizebox{\textwidth}{!}{%
\begin{tabular}{lrrrrrrr}
  \toprule
  & \multicolumn{2}{c}{Predictive fit} & \multicolumn{1}{c}{Mean uncertainty} & \multicolumn{2}{c}{Error ranking} & \multicolumn{2}{c}{Intervals} \\
  \cmidrule(lr){2-3}\cmidrule(lr){4-4}\cmidrule(lr){5-6}\cmidrule(lr){7-8}
  Method & MSE $\downarrow$ & NLL $\downarrow$ & $\Lmean$ $\downarrow$ & AUROC $\uparrow$ & SelEx $\downarrow$ & CovGap $\downarrow$ & Width$^\dagger$ \\
  \midrule
  \multicolumn{8}{l}{\emph{Model comparison}} \\
  Masked DER & $.619\pm.181$ & $-.251\pm.052$ & $.556\pm.430$ & $.827\pm.023$ & $.0895\pm.0470$ & $.0460\pm.0166$ & $18.93\pm35.84$ \\
  RN-DER & $\mathbf{.515\pm.059}$ & $\mathbf{-.256\pm.027}$ & $.315\pm.461$ & $.832\pm.009$ & $.0360\pm.0040$ & $.0512\pm.0085$ & $2.36\pm1.44$ \\
  Gaussian-MV & $.542\pm.042$ & $-.100\pm.050$ & $.577\pm.101$ & $\mathbf{.879\pm.017}$ & $\mathbf{.0289\pm.0025}$ & $.0934\pm.0020$ & $2.79\pm.30$ \\
  \textbf{MORSE} & $.533\pm.037$ & $-.168\pm.030$ & $\mathbf{.125\pm.109}$ & $.870\pm.007$ & $.0295\pm.0034$ & $\mathbf{.0176\pm.0086}$ & $\mathbf{1.45\pm.07}$ \\
  \midrule
  \multicolumn{8}{l}{\emph{Component and coupling ablations}} \\
  MORSE w/o ObsRef & $.544\pm.041$ & $1.418\pm.117$ & $.253\pm.123$ & $.866\pm.009$ & $.0364\pm.0040$ & $.1000\pm.0001$ & $14.89\pm1.96$ \\
  MORSE w/o Rank & $.541\pm.045$ & $-.106\pm.034$ & $.319\pm.160$ & $.874\pm.012$ & $.0292\pm.0032$ & $.0768\pm.0040$ & $2.00\pm.14$ \\
  MORSE-RC & $\mathbf{.533\pm.020}$ & $-.103\pm.035$ & $.255\pm.099$ & $\mathbf{.875\pm.009}$ & $\mathbf{.0281\pm.0021}$ & $.0761\pm.0044$ & $2.05\pm.10$ \\
  \textbf{MORSE} & $\mathbf{.533\pm.037}$ & $\mathbf{-.168\pm.030}$ & $\mathbf{.125\pm.109}$ & $.870\pm.007$ & $.0295\pm.0034$ & $\mathbf{.0176\pm.0086}$ & $\mathbf{1.45\pm.07}$ \\
  \bottomrule
  \end{tabular}}
\end{table*}

Table~\ref{tab:app-ablation-full} reports full ten-seed statistics for the
three core variants available in the frozen snapshot. Removing ObsRef destroys
scale control, while removing Rank retains useful error ordering but weakens
relative-factor separation and coverage.

\begin{table*}[!t]
\caption{Core MORSE ablations over ten paired seeds (mean $\pm$ sample standard
deviation; the same component cohort as Table~\ref{tab:mechanism}).}
\label{tab:app-ablation-full}
\centering
\setlength{\tabcolsep}{3.0pt}
\small
\resizebox{\textwidth}{!}{%
\begin{tabular}{lccccccc}
\toprule
Variant & MSE $\downarrow$ & NLL $\downarrow$ & $L_{\mathrm{rel}}$ $\downarrow$ & $L_{\mathrm{mean}}$ $\downarrow$ & Err. AUROC $\uparrow$ & Sel. excess $\downarrow$ & Cov. gap $\downarrow$ \\
\midrule
\textbf{MORSE}
& \textbf{.5331}$\pm$.0371 & \textbf{$-.1681$}$\pm$.0298 & .3424$\pm$.0906 & \textbf{.1253}$\pm$.1092 & .8705$\pm$.0070 & .0295$\pm$.0034 & \textbf{.0176}$\pm$.0086 \\
MORSE without ObsRef
& .5440$\pm$.0408 & 1.4182$\pm$.1172 & .8044$\pm$.1085 & .2531$\pm$.1233 & .8664$\pm$.0095 & .0364$\pm$.0040 & .1000$\pm$.0001 \\
MORSE without Rank
& .5411$\pm$.0450 & $-.1059\pm$.0338 & \textbf{.0114}$\pm$.0050 & .3191$\pm$.1597 & \textbf{.8737}$\pm$.0117 & \textbf{.0292}$\pm$.0032 & .0768$\pm$.0040 \\
\bottomrule
\end{tabular}}
\end{table*}

The two additional controls are reported in
Table~\ref{tab:app-extended-means}. The $\udata$-base control removes
observability conditioning from fitting while retaining supervision of the
relative factor. The Condition-only control retains conditioning but removes the explicit
ObsRef and Rank objectives. Their failures separate reconstruction conditioning
from semantic supervision.
\begin{table*}[!t]
\caption{Extended mechanism controls over ten paired seeds (mean $\pm$
sample standard deviation; the same component cohort as Table~\ref{tab:mechanism}).}
\label{tab:app-extended-means}
\centering
\small
\setlength{\tabcolsep}{2.2pt}
\resizebox{\textwidth}{!}{%
\begin{tabular}{lrrrrrrrr}
\toprule
Variant & MSE & NLL & $L_{\mathrm{rel}}$ & $L_{\mathrm{mean}}$ & AUROC & Sel. & Cov. gap & Width \\
\midrule
$\udata$ base
& .5474$\pm$.0266 & $-.1098\pm$.0406 & .1650$\pm$.0167 & .2632$\pm$.0519 & .8732$\pm$.0081 & .0293$\pm$.0027 & .0779$\pm$.0041 & 2.0594$\pm$.1027 \\
Condition only
& .5384$\pm$.0393 & 2.5312$\pm$.1932 & .4731$\pm$.3790 & .8549$\pm$.3924 & .8360$\pm$.0171 & .0729$\pm$.0319 & .1000$\pm$.0000 & 39.4471$\pm$10.3392 \\
\bottomrule
\end{tabular}}
\end{table*}

A separate 20-seed paired comparison evaluates MORSE with residual coupling
(MORSE-RC), which adds a direct residual-to-$\nu$ path. MORSE-RC improves MSE slightly and has
similar ranking, but it worsens NLL and sharply worsens coverage and interval
width. Accordingly, the residual-coupled variant provides no consistent external gain
that justifies a second direct teacher. Table~\ref{tab:app-rc} reports the
external metrics used for this choice; teacher-adjacent alignment metrics are
not used as the selection criterion.

\begin{table*}[!t]
\caption{MORSE versus residual-coupled MORSE-RC over 20 paired seeds
(mean $\pm$ sample standard deviation). Lower is better except Error AUROC.}
\label{tab:app-rc}
\centering
\small
\setlength{\tabcolsep}{3.5pt}
\begin{tabular}{lrr}
\toprule
Metric & MORSE-RC & MORSE \\
\midrule
MSE & .5320$\pm$.0198 & .5618$\pm$.0787 \\
NLL & $-.1110\pm$.0376 & $-.1519\pm$.0509 \\
Error AUROC & .8721$\pm$.0082 & .8692$\pm$.0125 \\
Selective excess & .0281$\pm$.0018 & .0302$\pm$.0044 \\
Coverage gap & .0762$\pm$.0035 & .0188$\pm$.0100 \\
Interval width & 2.0153$\pm$.1033 & 1.4804$\pm$.0708 \\
\bottomrule
\end{tabular}
\end{table*}

\subsection{Number of Views}

Figure~\ref{fig:app-kview} and Table~\ref{tab:app-kview} vary the number of
common-target views over ten paired seeds. More views lower reconstruction
error but do not improve error ranking, while training time grows almost
linearly; the final method fixes $K=4$.

\begin{figure*}[!t]
  \centering
  \includegraphics[width=.92\textwidth]{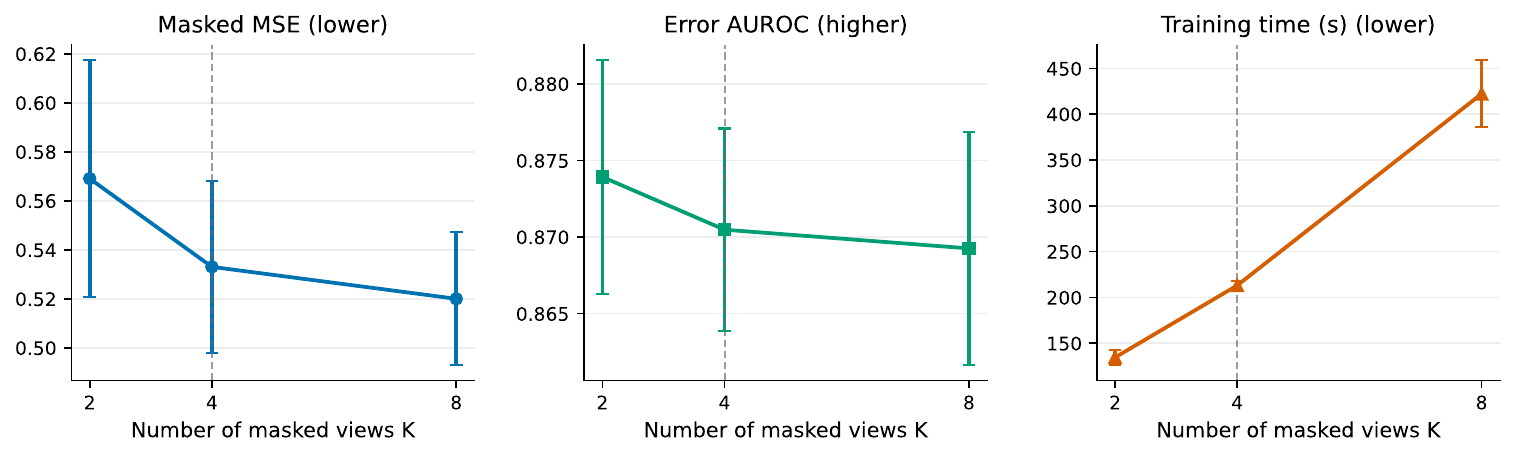}
  \caption{\textbf{Quality--cost trade-off for common-target views.} Every
  point and its error bar show mean $\pm$ standard deviation (divisor $n$) over ten paired
  seeds.
  Increasing $K$ reduces reconstruction error but does not improve error
  ranking, while cost grows almost linearly; $K=4$ is the prespecified
  compromise used throughout the paper.}
  \label{fig:app-kview}
\end{figure*}

\begin{table*}[!t]
\caption{Common-target view ablation over ten paired seeds (mean $\pm$ sample
standard deviation). Training time is
total wall-clock seconds per run. Forward passes per step include one
reference pass.}
\label{tab:app-kview}
\centering
\small
\setlength{\tabcolsep}{5pt}
\begin{tabular}{rrrrrr}
\toprule
$K$ & MSE & NLL & Error AUROC & train s/run & passes \\
\midrule
2 & .5691$\pm$.0510 & $-.1325\pm$.0246 & .8739$\pm$.0080 & 134.5$\pm$8.9 & 3 \\
4 & .5331$\pm$.0371 & $-.1681\pm$.0298 & .8705$\pm$.0070 & 213.1$\pm$4.8 & 5 \\
8 & .5200$\pm$.0286 & $-.1641\pm$.0588 & .8693$\pm$.0080 & 422.2$\pm$38.7 & 9 \\
\bottomrule
\end{tabular}
\end{table*}

\subsection{Structured Head Control}

To test whether the NIG-shaped factorization is merely redundant notation, an
explicit-head control predicts the mean, a positive data scale, and a positive
relative mean-uncertainty channel directly. It improves MSE slightly but weakens
error ranking and selective prediction. On this benchmark the two
parameterizations therefore trade reconstruction error against error ranking;
the comparison does not show that the NIG-shaped factorization is necessary,
and $\alpha$ and $\beta$ are not individually identified posterior parameters. Appendix~\ref{app:components} repeats the head comparison in a matched factorial on GP and pulse.

\begin{table*}[!t]
\caption{NIG-shaped MORSE versus explicit independent heads over ten paired
seeds.}
\label{tab:app-explicit}
\centering
\small
\setlength{\tabcolsep}{5pt}
\begin{tabular}{lrrrrrr}
\toprule
Head & MSE & NLL & Err. AUROC & Sel. & Cov. gap & Width\\
\midrule
NIG-shaped & .5331 & $-.1681$ & \textbf{.8705} & \textbf{.0295} & .0176 & 1.4485\\
Explicit & \textbf{.5140} & $-.1509$ & .8543 & .0402 & \textbf{.0094} & \textbf{1.4328}\\
\bottomrule
\end{tabular}
\end{table*}

\subsection{Finite-Training-Set Reference Audit}
\label{app:reference}

This subsection reports the original finite-training-set audit of the four
evidential methods; the extended comparison with ensembles is in
Appendix~\ref{app:pulse-audit}.

The strict reference follows the cross-training-set logic of
\citet{jurgens2024faithfully}. Fifty Gaussian first-order reconstruction models
are trained on independently sampled window datasets. On one fixed test and
mask bank, their predictive means define an empirical distribution with center
and variance $V_{\mathrm{ref}}$. The second-order methods are then trained on
five paired anchor datasets that are disjoint from the reference-model
datasets. MORSE never observes $V_{\mathrm{ref}}$ during training.

Each first-order learner uses the same two-layer Transformer hypothesis class
and 100-epoch optimization budget as the window methods, but outputs only a
reconstruction mean and one positive scalar variance per token and minimizes
Gaussian NLL. The same 50-model reference bank, deterministic test bank, and
fixed $K=4$ evaluation-mask bank are reused for all five anchors. Within anchor
$s$, DER, RN-DER, Gaussian-MV, and MORSE receive the identical serialized training
set. Saved content hashes verify this pairing, and reference-data seeds are
disjoint from anchor-data seeds. Only the anchor dataset changes across the
five paired repetitions.

Center root mean squared error (RMSE) compares a second-order predicted center
with the empirical reference center. Log-variance mean absolute error (MAE)
compares $\log\umean$ with
$\log V_{\mathrm{ref}}$. Wasserstein-1 compares the parametric predictive-mean
distribution with the empirical first-order model means. Coverage and width
refer to intervals over model means, not noisy observations. Clean-error AUROC
and area under the risk--coverage curve (AURC) are training-independent utility
checks, whereas Spearman correlation tests ordering without validating
magnitude.

\begin{table*}[!t]
\caption{Strict J\"urgens-aligned reference audit over five paired anchor
datasets (mean $\pm$ sample standard deviation). The empirical target is fixed and
independent: 50 Gaussian first-order models trained on resampled datasets.
Width ratio is ideal at one.}
\label{tab:app-reference-full}
\centering
\small
\setlength{\tabcolsep}{4pt}
\begin{tabular}{lrrrrr}
\toprule
Method & Center & Log-var & W1 & Cov. gap & Width\\
\midrule
DER & .2562$\pm$.0146 & 3.8535$\pm$.0897 & .3757$\pm$.0333 &
.0323$\pm$.0023 & 6.1212$\pm$.5051\\
RN-DER & .2307$\pm$.0132 & 4.2056$\pm$.1369 & .4418$\pm$.0269 &
.0471$\pm$.0012 & 8.3378$\pm$.5677\\
Gaussian-MV & .2354$\pm$.0814 & 4.6828$\pm$.1083 & .5190$\pm$.0526 &
.0480$\pm$.0008 & 10.5144$\pm$1.4901\\
\textbf{MORSE} & \textbf{.1584}$\pm$.0182 & \textbf{1.5496}$\pm$.0513 &
\textbf{.0785}$\pm$.0086 & .1203$\pm$.1085 &
\textbf{1.3480}$\pm$.1253\\
\bottomrule
\end{tabular}

\medskip
\begin{tabular}{lrrrr}
\toprule
Method & Noisy AUROC & Clean AUROC & AURC & Spearman\\
\midrule
DER & .8445$\pm$.0032 & .8469$\pm$.0240 & .0552$\pm$.0022 &
.8052$\pm$.0119\\
RN-DER & .8349$\pm$.0136 & .8699$\pm$.0352 & .0522$\pm$.0034 &
\textbf{.8606}$\pm$.0063\\
Gaussian-MV & \textbf{.8772}$\pm$.0069 & .9265$\pm$.0072 &
.0556$\pm$.0059 & .8309$\pm$.0096\\
\textbf{MORSE} & .8610$\pm$.0040 & \textbf{.9327}$\pm$.0345 &
\textbf{.0522}$\pm$.0015 & .8548$\pm$.0114\\
\bottomrule
\end{tabular}
\end{table*}

\begin{table*}[!t]
\caption{Paired MORSE improvement over each baseline across the five anchor
datasets. Intervals are paired 95\% bootstrap confidence intervals (CIs).
Positive differences
favor MORSE for all three lower-is-better metrics. The smallest attainable
two-sided exact sign-test $p$ value with five consistently signed pairs
is $1/16=.0625$; we therefore emphasize intervals and win counts.}
\label{tab:app-reference-paired}
\centering
\small
\setlength{\tabcolsep}{4pt}
\begin{tabular}{lccc}
\toprule
Baseline & Center benefit [95\% CI] & Log-var benefit [95\% CI] &
W1 benefit [95\% CI]\\
\midrule
DER & .0978 [.0776,.1186] & 2.3039 [2.1914,2.3870] &
.2972 [.2631,.3273]\\
RN-DER & .0723 [.0548,.0898] & 2.6560 [2.5830,2.7289] &
.3632 [.3473,.3794]\\
Gaussian-MV & .0770 [.0011,.1515] & 3.1332 [3.0243,3.2521] &
.4405 [.4047,.4840]\\
\bottomrule
\end{tabular}
\end{table*}

MORSE wins log-variance error and W1 on all five anchors against every baseline, and
Center RMSE on all five anchors against DER and RN-DER and four of five against
Gaussian-MV. DER retains the smallest nominal coverage gap, while MORSE has the
largest and most variable coverage gap despite its near-reference width ratio.
RN-DER retains the highest Spearman correlation. These exceptions are reported
because internal ordering, interval width, and nominal coverage do not validate
one another.

The finite-training-set audit across training sizes is reported with the
proxy-based training-size audit in Appendix~\ref{app:pulse-audit}.

\subsection{Matched Component and Head Study}
\label{app:components}

\paragraph{Design.}
We cross the three components, Condition (C), Bound/ObsRef (B), and Order/Rank (O), in a full $2\times2\times2$ factorial with two heads: the four-output NIG head and a three-output explicit head that predicts the data scale and the relative factor directly. Four further controls use all components: $\alpha$ fixed to 2, a head that outputs the predictive-mean variance directly, an unnormalized teacher that supervises $\urel$ with the raw cross-view variance instead of $\rnorm$, and residual coupling (MORSE-RC). Each of the 20 variants is trained on five GP anchors and ten pulse seeds, 300 fits in total. This GP benchmark fixes the noise boundary at .5, so the oracle and the learner face the same noise information, and its contextual reference uses the 124 recorded masks that remain after removing the four scoring contexts; its MORSE error (.7349) is therefore not comparable with Table~\ref{tab:gp-bayes}. The pulse cohort redraws training masks every epoch and normalizes on training data only. For every head, NLL uses the Gaussian variance $\udata+\umean$. The primary endpoints are GP contextual LogVar MAE and pulse NLL, with six declared comparisons per domain and Holm correction. The settings were fixed after earlier results were available, and a single model initializer is used.

\begin{table*}[!htbp]
\caption{Matched component and head study on the fixed-boundary GP benchmark (five paired anchors; mean $\pm$ sample SD). C, B, and O switch Condition, Bound (ObsRef), and Order (Rank) on or off. LogVar and Spearman compare the four-context score with the 124-mask oracle response; NLL uses the Gaussian variance $\udata+\umean$ for every head. Bold marks the best mean; width is not ranked.}
\label{tab:components-gp}
\centering
\resizebox{\textwidth}{!}{%
\begin{tabular}{lrrrrrr}
\toprule
Variant & LogVar & Spearman & Center RMSE & NLL & CovGap & Width \\
 & $\downarrow$ & $\uparrow$ & $\downarrow$ & $\downarrow$ & $\downarrow$ &  \\
\midrule
NIG C0B0O0 & 2.0910$\pm$.1165 & .5932$\pm$.0074 & .1395$\pm$.0018 & .1702$\pm$.0214 & .0806$\pm$.0031 & 1.4128$\pm$.0485 \\
NIG C0B0O1 & 2.0641$\pm$.0798 & .6246$\pm$.0041 & .1416$\pm$.0047 & .1701$\pm$.0159 & .0809$\pm$.0022 & 1.3938$\pm$.0273 \\
NIG C0B1O0 & .7135$\pm$.0081 & .5878$\pm$.0070 & .1425$\pm$.0056 & .0601$\pm$.0090 & .0291$\pm$.0045 & 1.0448$\pm$.0168 \\
NIG C0B1O1 & .7313$\pm$.0118 & .6215$\pm$.0034 & .1408$\pm$.0061 & .0595$\pm$.0107 & .0348$\pm$.0042 & 1.0669$\pm$.0191 \\
NIG C1B0O0 & 1.9900$\pm$.1392 & .5933$\pm$.0065 & .1434$\pm$.0049 & .1497$\pm$.0302 & .0747$\pm$.0046 & 1.3445$\pm$.0588 \\
NIG C1B0O1 & 1.9544$\pm$.1345 & .6247$\pm$.0029 & .1382$\pm$.0015 & .1396$\pm$.0167 & .0744$\pm$.0063 & 1.3342$\pm$.0599 \\
NIG C1B1O0 & \textbf{.7097}$\pm$.0063 & .5873$\pm$.0063 & .1388$\pm$.0077 & \textbf{.0437}$\pm$.0125 & .0221$\pm$.0010 & 1.0068$\pm$.0138 \\
\textbf{NIG C1B1O1 (MORSE)} & .7349$\pm$.0120 & .6215$\pm$.0025 & \textbf{.1379}$\pm$.0029 & .0464$\pm$.0065 & .0232$\pm$.0056 & 1.0198$\pm$.0191 \\
\midrule
Explicit C0B0O0 & 2.1103$\pm$.0823 & .5865$\pm$.0093 & .1459$\pm$.0039 & .1834$\pm$.0243 & .0821$\pm$.0025 & 1.4472$\pm$.0390 \\
Explicit C0B0O1 & 2.3762$\pm$.0375 & .6193$\pm$.0035 & .1444$\pm$.0020 & .2240$\pm$.0138 & .0873$\pm$.0011 & 1.5100$\pm$.0233 \\
Explicit C0B1O0 & .7197$\pm$.0110 & .5600$\pm$.0210 & .1410$\pm$.0056 & .0693$\pm$.0179 & .0303$\pm$.0012 & 1.0506$\pm$.0110 \\
Explicit C0B1O1 & .7157$\pm$.0114 & .6204$\pm$.0033 & .1390$\pm$.0057 & .0633$\pm$.0137 & .0351$\pm$.0041 & 1.0655$\pm$.0259 \\
Explicit C1B0O0 & 1.9857$\pm$.0693 & .5894$\pm$.0160 & .1430$\pm$.0038 & .1487$\pm$.0207 & .0753$\pm$.0031 & 1.3677$\pm$.0301 \\
Explicit C1B0O1 & 2.2765$\pm$.0595 & .6198$\pm$.0042 & .1424$\pm$.0036 & .1939$\pm$.0176 & .0838$\pm$.0022 & 1.4450$\pm$.0271 \\
Explicit C1B1O0 & .7196$\pm$.0093 & .5688$\pm$.0156 & .1393$\pm$.0050 & .0585$\pm$.0143 & .0269$\pm$.0052 & 1.0316$\pm$.0162 \\
Explicit C1B1O1 & .7249$\pm$.0105 & .6199$\pm$.0026 & .1419$\pm$.0011 & .0654$\pm$.0075 & .0282$\pm$.0036 & 1.0442$\pm$.0158 \\
\midrule
Fixed $\alpha=2$ & .7308$\pm$.0123 & .6246$\pm$.0057 & .1421$\pm$.0051 & .0549$\pm$.0040 & .0227$\pm$.0050 & 1.0238$\pm$.0205 \\
Direct mean variance & .7311$\pm$.0104 & \textbf{.6272}$\pm$.0039 & .1492$\pm$.0032 & .0676$\pm$.0045 & .0216$\pm$.0076 & 1.0311$\pm$.0283 \\
Raw teacher & 1.6972$\pm$.0390 & .6192$\pm$.0037 & .1383$\pm$.0030 & .0445$\pm$.0082 & \textbf{.0165}$\pm$.0053 & .9929$\pm$.0170 \\
MORSE-RC & 2.0748$\pm$.0337 & .6223$\pm$.0055 & .1386$\pm$.0057 & .1550$\pm$.0099 & .0785$\pm$.0020 & 1.3660$\pm$.0143 \\
\bottomrule
\end{tabular}}
\end{table*}

\begin{table*}[!htbp]
\caption{Matched component and head study on the pulse benchmark (ten seeds; mean $\pm$ sample SD). This cohort is not mixed with Tables~\ref{tab:app-backbone-full}--\ref{tab:app-extended-means}. NLL uses the Gaussian variance $\udata+\umean$ for every head. Bold marks the best mean; width is not ranked.}
\label{tab:components-pulse}
\centering
\resizebox{\textwidth}{!}{%
\begin{tabular}{lrrrrrrr}
\toprule
Variant & MSE & NLL & AUROC & AURC & CovGap & Width & $L_{\mathrm{mean}}$ \\
 & $\downarrow$ & $\downarrow$ & $\uparrow$ & $\downarrow$ & $\downarrow$ &  & $\downarrow$ \\
\midrule
NIG C0B0O0 & .5137$\pm$.0326 & $-.0565\pm$.0714 & .8715$\pm$.0074 & .0543$\pm$.0030 & .0694$\pm$.0062 & 2.4342$\pm$.1679 & .1608$\pm$.0765 \\
NIG C0B0O1 & .5183$\pm$.0366 & $-.0538\pm$.0416 & \textbf{.8792}$\pm$.0079 & .0553$\pm$.0043 & .0803$\pm$.0050 & 2.1985$\pm$.2129 & .2314$\pm$.1294 \\
NIG C0B1O0 & .5115$\pm$.0294 & $-.1705\pm$.0419 & .8748$\pm$.0114 & .0534$\pm$.0030 & .0149$\pm$.0090 & 1.4983$\pm$.0932 & .2773$\pm$.1708 \\
NIG C0B1O1 & .5143$\pm$.0242 & $-.1550\pm$.0664 & .8758$\pm$.0064 & .0542$\pm$.0025 & .0191$\pm$.0080 & 1.4646$\pm$.0627 & .1402$\pm$.0827 \\
NIG C1B0O0 & .5265$\pm$.0181 & $-.0571\pm$.0445 & .8761$\pm$.0088 & .0557$\pm$.0026 & .0707$\pm$.0069 & 2.5693$\pm$.3079 & .1911$\pm$.1393 \\
NIG C1B0O1 & .5030$\pm$.0256 & $-.0831\pm$.0352 & .8731$\pm$.0057 & .0532$\pm$.0022 & .0798$\pm$.0053 & 2.0578$\pm$.1394 & .2815$\pm$.1012 \\
NIG C1B1O0 & .4961$\pm$.0358 & $-.1735\pm$.0309 & .8733$\pm$.0071 & .0530$\pm$.0029 & .0126$\pm$.0080 & 1.4676$\pm$.0968 & .2508$\pm$.1637 \\
\textbf{NIG C1B1O1 (MORSE)} & .5101$\pm$.0323 & $-.1837\pm$.0210 & .8722$\pm$.0105 & .0537$\pm$.0032 & .0109$\pm$.0072 & 1.4317$\pm$.0599 & \textbf{.1086}$\pm$.1071 \\
\midrule
Explicit C0B0O0 & \textbf{.4882}$\pm$.0384 & $-.1148\pm$.0342 & .8728$\pm$.0071 & \textbf{.0518}$\pm$.0032 & .0717$\pm$.0073 & 1.9160$\pm$.0878 & .3483$\pm$.1042 \\
Explicit C0B0O1 & .5153$\pm$.0342 & $-.0655\pm$.0362 & .8753$\pm$.0079 & .0560$\pm$.0033 & .0810$\pm$.0030 & 2.1485$\pm$.1368 & .2792$\pm$.0971 \\
Explicit C0B1O0 & .4919$\pm$.0300 & $-.1957\pm$.0349 & .8524$\pm$.0099 & .0609$\pm$.0048 & .0076$\pm$.0055 & 1.3846$\pm$.0579 & .4495$\pm$.2551 \\
Explicit C0B1O1 & .5071$\pm$.0288 & $-.1852\pm$.0290 & .8586$\pm$.0099 & .0644$\pm$.0038 & .0125$\pm$.0048 & 1.4439$\pm$.0766 & .2178$\pm$.1071 \\
Explicit C1B0O0 & .4947$\pm$.0459 & $-.1201\pm$.0337 & .8705$\pm$.0091 & .0525$\pm$.0038 & .0704$\pm$.0061 & 1.8761$\pm$.1276 & .3223$\pm$.1464 \\
Explicit C1B0O1 & .5018$\pm$.0309 & $-.0760\pm$.0462 & .8701$\pm$.0079 & .0545$\pm$.0039 & .0798$\pm$.0050 & 2.0581$\pm$.1628 & .2313$\pm$.1137 \\
Explicit C1B1O0 & .5141$\pm$.0570 & $-.1723\pm$.0605 & .8612$\pm$.0109 & .0663$\pm$.0096 & .0121$\pm$.0074 & 1.4689$\pm$.1170 & .3996$\pm$.2293 \\
Explicit C1B1O1 & .4942$\pm$.0257 & \textbf{$-.1966$}$\pm$.0318 & .8584$\pm$.0076 & .0643$\pm$.0040 & \textbf{.0072}$\pm$.0040 & 1.4304$\pm$.0378 & .2164$\pm$.2445 \\
\midrule
Fixed $\alpha=2$ & .4941$\pm$.0348 & $-.1687\pm$.0544 & .8690$\pm$.0065 & .0531$\pm$.0035 & .0151$\pm$.0073 & 1.4221$\pm$.0455 & .2021$\pm$.1203 \\
Direct mean variance & .4999$\pm$.0269 & $-.1702\pm$.0316 & .8515$\pm$.0108 & .0631$\pm$.0027 & .0090$\pm$.0085 & 1.4115$\pm$.0584 & .2443$\pm$.1461 \\
Raw teacher & .5048$\pm$.0320 & $-.1855\pm$.0348 & .8750$\pm$.0053 & .0538$\pm$.0038 & .0152$\pm$.0051 & 1.4642$\pm$.0943 & .3786$\pm$.3094 \\
MORSE-RC & .5050$\pm$.0220 & $-.0870\pm$.0359 & .8747$\pm$.0075 & .0529$\pm$.0023 & .0770$\pm$.0049 & 2.0331$\pm$.1325 & .2813$\pm$.0667 \\
\bottomrule
\end{tabular}}
\end{table*}

\begin{table}[!htbp]
\caption{Declared primary comparisons in the component study: full MORSE minus each control (negative favors MORSE), paired 95\% intervals from 20,000 whole-anchor or whole-seed draws, wins, and Holm-adjusted exact sign-flip $p$ within each domain. The five GP anchors are descriptive.}
\label{tab:components-paired}
\centering
\small
\begin{tabular}{llccc}
\toprule
Endpoint & Control & Difference [95\% interval] & Wins & Holm $p$ \\
\midrule
GP (LogVar) & no Condition & $+$0.00356 [$-$0.00845, $+$0.01383] & 1/5 & 1.000 \\
GP (LogVar) & no Bound & $-$1.21951 [$-$1.32962, $-$1.11433] & 5/5 & 0.375 \\
GP (LogVar) & no Order & $+$0.02515 [$+$0.01646, $+$0.03384] & 0/5 & 0.375 \\
GP (LogVar) & explicit head & $+$0.00994 [$+$0.00080, $+$0.01802] & 1/5 & 0.750 \\
GP (LogVar) & fixed $\alpha=2$ & $+$0.00411 [$-$0.00470, $+$0.01234] & 1/5 & 1.000 \\
GP (LogVar) & direct mean variance & $+$0.00379 [$-$0.00625, $+$0.01383] & 2/5 & 1.000 \\
\midrule
Pulse (NLL) & no Condition & $-$0.02863 [$-$0.06976, $+$0.00845] & 6/10 & 1.000 \\
Pulse (NLL) & no Bound & $-$0.10055 [$-$0.12413, $-$0.07871] & 10/10 & 0.012 \\
Pulse (NLL) & no Order & $-$0.01011 [$-$0.02622, $+$0.00600] & 5/10 & 1.000 \\
Pulse (NLL) & explicit head & $+$0.01298 [$-$0.00524, $+$0.03242] & 5/10 & 1.000 \\
Pulse (NLL) & fixed $\alpha=2$ & $-$0.01501 [$-$0.05050, $+$0.01475] & 6/10 & 1.000 \\
Pulse (NLL) & direct mean variance & $-$0.01342 [$-$0.03901, $+$0.00644] & 7/10 & 1.000 \\
\bottomrule
\end{tabular}
\end{table}

\paragraph{Findings.}
Bound has the clearest effect: removing it raises GP scale error from .7349 to 1.9544 and pulse NLL from $-.1837$ to $-.0831$, with the full model better on all ten pulse seeds (Holm $p=.012$; Table~\ref{tab:components-paired}). Restoring the residual path raises GP scale error to 2.0748, and the unnormalized teacher to 1.6972, so decoupling and normalization by the data scale both matter. Condition improves GP NLL and coverage but not scale error (.7313 without it), and Order improves ranking (Spearman .6215 versus .5873) at a small cost in scale error (.7097 without it). The NIG head is not necessary for contextual scale alignment: the explicit head (.7249), fixed $\alpha$ (.7308), and direct mean variance (.7311) align as well. On pulse, the NIG head gives better error AUROC and AURC than the explicit and direct-variance heads, whereas the explicit head has lower mean MSE, NLL, and coverage gap. The study therefore supports distinct roles for the components, not the necessity of each component or of the four NIG outputs.

\section{Pulse Proxy and Training-Size Audit}
\label{app:pulse-audit}

\paragraph{Reference construction.}
Exact Bayes inference is not practical for the pulse generator: self-normalized importance sampling over its five latent parameters reached the pre-specified minimum effective sample size of 200 over all 128 masks in none of 64 test windows (median window-level ESS 9.45). We therefore approximate the population-optimal predictor with first-order networks trained on freshly generated windows at every step (unlimited data). Eight proxies use the standard two-layer, width-64 architecture and four use a four-layer, width-128 architecture; all use exponential moving averages of the weights (decay .9995), a cosine schedule to a final learning rate of $10^{-6}$, and batch size 256. For each architecture, the ensemble mean over seeds is evaluated on 128 fixed masks per common target. We subtract the seed-noise contribution $\widehat N(t)=S^{-1}\,\mathrm{mean}_{s<s'}\Var_j(\mu_{s,j,t}-\mu_{s',j,t})/2$ from the cross-mask variance of the ensemble mean, which gives the corrected reference $R^\ast_c(t)$, and call a target resolvable if $R^\ast_c(t)\geq10\,\widehat N(t)$. Resolvability is defined from proxy quantities only.

\paragraph{Quality checks.}
A split of the standard seeds into two halves gives corrected references that agree to LogVar MAE .117 on resolvable targets (threshold .2). The two architectures differ by .241 (threshold .2), and the Bayes check could not arbitrate, so the larger architecture was selected by its lower validation clean-target MSE (.19960 versus .20312), as specified in advance for this case. Its ensemble mean has clean-target MSE .19930, below the best of the 160 finite-$N$ method anchors (.22550), and 99.47\% of common targets are resolvable (32,594 targets). These checks establish a usable operational reference; they do not show that it equals the Bayes predictor.

\paragraph{Training-size results.}
Table~\ref{tab:pulse-n} compares the 128-mask response of each MORSE predictor with $R^\ast_c$ on the resolvable targets. The predictor's error decreases at every step and improves from $N=512$ to $32{,}768$ on all ten paired anchors. The learned $\umean$ follows the same trend at a much larger error. The mean predictor response rises from .009 at $N=512$ to .040 at $N=32{,}768$ and passes the fixed reference mean of .034, while the finite-training-set reference decreases; the early deficit is consistent with the shrinkage mechanism described after Proposition~\ref{prop:population}.

\begin{table}[!t]
\caption{Pulse training-size audit on the 32,594 resolvable common targets (ten paired anchors per $N$). The first column is the 128-mask response of each MORSE predictor, $\robs(\theta)$, against the corrected proxy reference $R^\ast_c$ (mean $\pm$ sample standard deviation); the second is MORSE's learned $\umean$ against the same reference. The mean of $R^\ast_c$ is .03395 at every $N$. The last column is the finite-training-set reference on the same targets.}
\label{tab:pulse-n}
\centering
\small
\begin{tabular}{rcccc}
\toprule
$N$ & $\robs(\theta)$ vs $R^\ast_c$ LogVar $\downarrow$ & $\umean$ vs $R^\ast_c$ LogVar $\downarrow$ & Mean $\robs(\theta)$ & Mean $\vref$\\
\midrule
512 & 1.987$\pm$.187 & 4.477 & .009 & .0427\\
2048 & 1.562$\pm$.092 & 3.327 & .022 & .0445\\
8192 & 1.281$\pm$.091 & 2.345 & .038 & .0361\\
32768 & .917$\pm$.087 & 1.687 & .040 & .0200\\
\bottomrule
\end{tabular}
\end{table}

\paragraph{Checkpoint provenance.}
The original synthetic runs did not retain model checkpoints. For this audit, models were retrained with the frozen protocol and seeds, and a retrained model was accepted as equivalent only if its fixed-bank predictions matched the archived ones to $10^{-6}$ MAE. Some $N=32{,}768$ anchors could not be reproduced exactly on the available hardware and were retrained on a different GPU type; the audit uses these retrained models consistently, so their $\vref$-based values can differ slightly from Table~\ref{tab:app-n-full}.

\paragraph{Finite-training-set audit across sizes.}
The original finite-training-set construction was repeated at seven training sizes, each with an independent bank of 50 first-order models and ten paired second-order anchors (Figure~\ref{fig:app-n-sensitivity} and Table~\ref{tab:app-n-full}). Training uses 100 epochs at every size, so larger datasets receive proportionally more gradient updates. MORSE has the lowest LogVar MAE to $\vref$ at every size and its error decreases at all six transitions, whereas baseline magnitudes vary sharply and non-monotonically. The pooled MORSE mean contracts from .084 to .014 at $N=8{,}192$ and then rises to .022 at $N=32{,}768$ (a 55\% increase), while $\vref$ keeps decreasing; the close match of the two means at $N=32{,}768$ is therefore a crossing of opposite trends rather than convergence, as Proposition~\ref{prop:population} predicts.

\begin{figure}[!t]
  \centering
  \includegraphics[width=.85\textwidth]{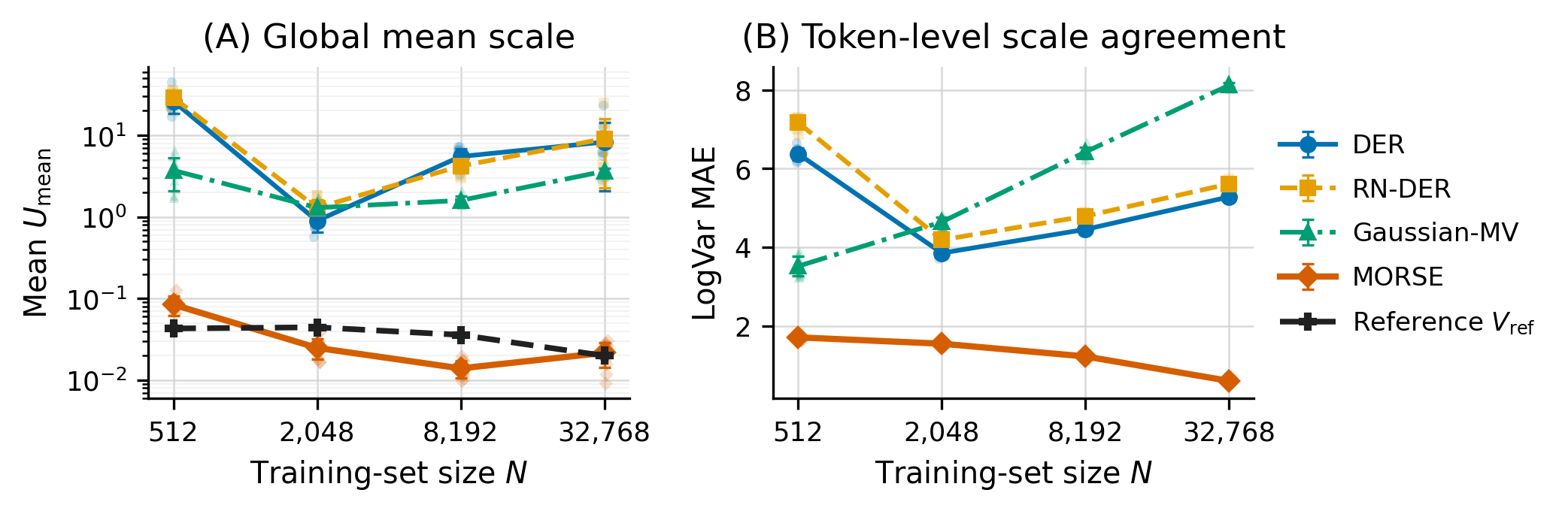}
  \caption{Finite-training-set audit with ten paired anchors and an independent 50-model reference bank per $N$. (A) Mean $\umean$ and mean $\vref$. (B) Token-level log-variance MAE to $\vref$; lower is better. Curves show mean $\pm$ standard deviation; pale markers show individual anchors.}
  \label{fig:app-n-sensitivity}
\end{figure}

\begin{table}[!t]
\caption{Finite-training-set audit across seven training sizes. Panel A: ten paired second-order anchors per size (mean $\pm$ sample standard deviation); mean $\umean$ is descriptive. LogVar MAE is the common-target average of $|\log(\umean+\epsilon)-\log(\vref+\epsilon)|$. Panel B: empirical reference variance; each row uses an independent bank of 50 first-order models on the same fixed test and mask bank.}
\label{tab:app-n-full}
\centering
\footnotesize
\setlength{\tabcolsep}{2.5pt}
\resizebox{\linewidth}{!}{%
\begin{tabular}{rcccccccc}
\toprule
& \multicolumn{4}{c}{A. Mean $\umean$} & \multicolumn{4}{c}{A. LogVar MAE to $\vref$ $\downarrow$}\\
\cmidrule(lr){2-5}\cmidrule(lr){6-9}
$N$ & DER & RN-DER & Gaussian-MV & MORSE & DER & RN-DER & Gaussian-MV & MORSE\\
\midrule
512 & 26.21$\pm$7.65 & 28.90$\pm$5.43 & 3.72$\pm$1.61 & .084$\pm$.022 & 6.371$\pm$.153 & 7.166$\pm$.139 & 3.517$\pm$.245 & \textbf{1.708}$\pm$.061\\
1024 & 411.22$\pm$227.15 & 307.81$\pm$236.56 & 1.542$\pm$.351 & .0479$\pm$.0065 & 7.719$\pm$.529 & 6.315$\pm$1.161 & 3.893$\pm$.137 & \textbf{1.564}$\pm$.079\\
2048 & .897$\pm$.252 & 1.293$\pm$.340 & 1.297$\pm$.286 & .0249$\pm$.0069 & 3.850$\pm$.085 & 4.186$\pm$.111 & 4.641$\pm$.119 & \textbf{1.546}$\pm$.040\\
4096 & .815$\pm$.180 & .991$\pm$.202 & 1.585$\pm$.302 & .0159$\pm$.0032 & 4.069$\pm$.122 & 4.434$\pm$.061 & 5.759$\pm$.226 & \textbf{1.505}$\pm$.038\\
8192 & 5.50$\pm$1.39 & 4.21$\pm$.77 & 1.592$\pm$.202 & .0140$\pm$.0033 & 4.452$\pm$.073 & 4.788$\pm$.051 & 6.421$\pm$.118 & \textbf{1.223}$\pm$.030\\
16384 & 6.73$\pm$4.15 & 6.40$\pm$4.65 & 2.627$\pm$.229 & .0147$\pm$.0037 & 4.804$\pm$.060 & 5.195$\pm$.047 & 7.331$\pm$.078 & \textbf{.863}$\pm$.048\\
32768 & 8.28$\pm$6.18 & 9.07$\pm$6.80 & 3.629$\pm$.268 & .0217$\pm$.0074 & 5.278$\pm$.058 & 5.609$\pm$.054 & 8.116$\pm$.058 & \textbf{.600}$\pm$.046\\
\bottomrule
\end{tabular}%
}

\medskip
\begin{tabular}{rccc}
\toprule
$N$ & B. Mean $\vref$ & Median & 90th percentile\\
\midrule
512 & .04305 & .000713 & .17473\\
1024 & .04812 & .000515 & .18551\\
2048 & .04449 & .000406 & .14018\\
4096 & .03994 & .000270 & .08216\\
8192 & .03591 & .000203 & .06683\\
16384 & .02501 & .000130 & .03738\\
32768 & .01992 & .000088 & .03799\\
\bottomrule
\end{tabular}
\end{table}

\paragraph{Data-scale calibration by noise region.}
The generator provides the normalized local noise variance of every target: .00166 in the low-noise region and .16614 in the high-noise region. Averaging $|\log(\overline U_{\mathrm{data},t}+\epsilon)-\log(\sigma_t^2+\epsilon)|$ over each region at $N=32{,}768$, MORSE obtains low/high-noise errors of 1.053/.294, compared with 1.089/.306 for Gaussian-MV, 1.461/4.959 for DER, and 3.680/5.152 for RN-DER. This separates the calibration of $\udata$ from the reference alignment of $\umean$.

\paragraph{Extended comparison with ensembles.}
Table~\ref{tab:pulse-extended} adds five-member deep and bootstrap ensembles to the $N=2048$ pulse comparison on the resolvable targets. Both ensembles are closer than MORSE's learned score to the contextual proxy and to $\vref$. Ensembles are closer on this proxy, so the contextual advantage observed on GP is benchmark-dependent; the ensembles also use five separately trained models per anchor.

\begin{table*}[t]
\caption{Pulse proxy audit on five paired anchors and 32,594 resolvable targets (mean $\pm$ sample SD). Context variance uses the selected four-model large-architecture proxy, not an exact Bayes reference. The independent training-set reference contains 50 first-order models. Both ensemble controls outperform MORSE in LogVar MAE to both targets in this benchmark.}
\label{tab:pulse-extended}
\centering
\small
\setlength{\tabcolsep}{4pt}
\begin{tabular}{lcccc}
\toprule
Method & Context LogVar $\downarrow$ & Context $\rho$ $\uparrow$ & $V_{\mathrm{ref}}$ LogVar $\downarrow$ & $V_{\mathrm{ref}}$ $\rho$ $\uparrow$ \\
\midrule
DER & $6.140 \pm 0.096$ & $0.458 \pm 0.008$ & $3.847 \pm 0.089$ & $0.806 \pm 0.012$ \\
RN-DER & $6.541 \pm 0.143$ & $0.464 \pm 0.009$ & $4.201 \pm 0.136$ & $0.861 \pm 0.006$ \\
Gaussian-MV & $7.040 \pm 0.104$ & $\mathbf{0.464 \pm 0.016}$ & $4.678 \pm 0.108$ & $0.831 \pm 0.010$ \\
MORSE & $3.332 \pm 0.037$ & $0.428 \pm 0.015$ & $1.545 \pm 0.051$ & $0.855 \pm 0.011$ \\
MORSE-IS & $3.526 \pm 0.052$ & $0.447 \pm 0.010$ & $1.485 \pm 0.073$ & $\mathbf{0.867 \pm 0.005}$ \\
Deep ensemble (5) & $\mathbf{2.555 \pm 0.210}$ & $0.429 \pm 0.057$ & $\mathbf{0.722 \pm 0.100}$ & $0.751 \pm 0.109$ \\
Bootstrap ensemble (5) & $2.661 \pm 0.210$ & $0.372 \pm 0.038$ & $0.748 \pm 0.108$ & $0.609 \pm 0.246$ \\
\bottomrule
\end{tabular}
\end{table*}

\section{Gaussian-Process Reference Audits}
\label{app:gp}

\subsection{Default Setting}

\paragraph{Generator and oracle.}
Length-128 windows come from a zero-mean Mat\'ern-$5/2$ Gaussian process with unit signal variance and length scale .05. Independent noise has standard deviation .05 or .30, separated by a boundary sampled uniformly from $[.2,.8]$. A separate calibration bank of 20,000 windows defines one affine normalization shared by every model and by the oracle. The Bayes oracle computes the exact conditional mean at hidden positions from the known covariance kernel and local noise variances~\citep{rasmussen2006gaussian}; the learned models receive only their masked observations and positions. Because the oracle conditions on each window's true noise variances, including the boundary, it is a known-noise Bayes target. The population-optimal predictor $\mu^\ast$ of Proposition~\ref{prop:population} for models that see only the visible values would instead average the conditional mean over the boundary's posterior. All compared models receive the same information, so this distinction concerns what the reference represents rather than the inputs available to any one method. The fixed test bank contains 1,024 windows and 32,768 common targets. The variance of the 128 exact conditional means is a finite-mask approximation to the population contextual response.

\paragraph{Training protocol.}
Five paired anchor datasets contain 2,048 windows each. All models use the frozen two-layer, width-64, four-head Transformer, AdamW with learning rate $10^{-3}$ and weight decay $10^{-4}$, batch size 64, and 100 epochs, with four views that mask half of the tokens and a shared-target rate of .25. Deep ensembles vary the initialization on each anchor; bootstrap ensembles resample each anchor with replacement and fix the initialization. MC dropout selects its probability from $\{.05,.10,.20\}$ by Gaussian NLL on a deterministic 10\% holdout of the anchor and then refits on the complete anchor; four anchors select .10 and one selects .20. The selection uses no test or reference quantities.

\paragraph{Two variation axes.}
The contextual reference varies the mask while fixing the predictor. The training-set reference $\vref$ comes from 20 first-order models trained on disjoint datasets: at each fixed evaluation context we compute the cross-model variance and then average over the four shared contexts. Ensemble and dropout scores use the same operation on their model or stochastic-pass axis, and evidential scores average $\umean$ over the same four contexts, so no table mixes model variance with variance caused by changing masks.

\begin{table*}[t]
\caption{GP two-reference audit on five paired anchors (mean $\pm$ sample SD). Center is RMSE of first-order conditional means against analytic Bayes means. The same four evaluation contexts are used for every learned score and for the 20-model independent $V_{\mathrm{ref}}$ bank; model-axis variances are computed before averaging over contexts. Raw rows use 4 or 128 masks and estimate context variation directly. The two reference quantities measure different sources of variation. The no-attention MC dropout row is retrained without attention-probability dropout (Appendix~\ref{app:mc-sites}).}
\label{tab:gp-two-reference}
\centering
\small
\setlength{\tabcolsep}{4pt}
\resizebox{\linewidth}{!}{%
\begin{tabular}{lccccc}
\toprule
Method & Center RMSE $\downarrow$ & Context LogVar $\downarrow$ & Context $\rho$ $\uparrow$ & $V_{\mathrm{ref}}$ LogVar $\downarrow$ & $V_{\mathrm{ref}}$ $\rho$ $\uparrow$ \\
\midrule
DER & $0.135 \pm 0.004$ & $2.150 \pm 0.105$ & $0.526 \pm 0.006$ & $3.432 \pm 0.105$ & $0.486 \pm 0.009$ \\
RN-DER & $\mathbf{0.129 \pm 0.004}$ & $2.414 \pm 0.071$ & $0.525 \pm 0.006$ & $3.704 \pm 0.070$ & $0.481 \pm 0.006$ \\
Gaussian-MV & $0.139 \pm 0.008$ & $4.336 \pm 0.111$ & $0.506 \pm 0.017$ & $5.632 \pm 0.111$ & $0.455 \pm 0.021$ \\
Deep ensemble (5) & $0.132 \pm 0.002$ & $2.166 \pm 0.140$ & $0.303 \pm 0.046$ & $0.915 \pm 0.091$ & $0.680 \pm 0.052$ \\
Bootstrap ensemble (5) & $0.129 \pm 0.002$ & $1.607 \pm 0.092$ & $0.303 \pm 0.055$ & $\mathbf{0.548 \pm 0.027}$ & $0.697 \pm 0.035$ \\
MC dropout (50) & $0.133 \pm 0.002$ & $0.807 \pm 0.013$ & $0.467 \pm 0.026$ & $1.185 \pm 0.107$ & $\mathbf{0.797 \pm 0.017}$ \\
\quad no attention (50) & $0.135 \pm 0.007$ & $1.012 \pm 0.065$ & $0.416 \pm 0.055$ & $0.732 \pm 0.070$ & $0.728 \pm 0.029$ \\
MORSE & $0.144 \pm 0.005$ & $0.774 \pm 0.022$ & $0.537 \pm 0.008$ & $1.499 \pm 0.098$ & $0.553 \pm 0.035$ \\
MORSE-IS & $0.144 \pm 0.006$ & $1.109 \pm 0.067$ & $0.539 \pm 0.008$ & $2.263 \pm 0.086$ & $0.550 \pm 0.034$ \\
MORSE raw (4) & $0.144 \pm 0.005$ & $1.212 \pm 0.007$ & $0.474 \pm 0.010$ & $1.512 \pm 0.027$ & $0.443 \pm 0.007$ \\
RN-DER raw (4) & $\mathbf{0.129 \pm 0.004}$ & $1.193 \pm 0.005$ & $0.498 \pm 0.007$ & $1.493 \pm 0.016$ & $0.429 \pm 0.006$ \\
MORSE raw (128) & $0.144 \pm 0.005$ & $0.673 \pm 0.032$ & $0.718 \pm 0.016$ & $1.682 \pm 0.044$ & $0.605 \pm 0.009$ \\
RN-DER raw (128) & $\mathbf{0.129 \pm 0.004}$ & $\mathbf{0.615 \pm 0.024}$ & $\mathbf{0.756 \pm 0.013}$ & $1.635 \pm 0.033$ & $0.583 \pm 0.011$ \\
\bottomrule
\end{tabular}%
}
\end{table*}

MORSE has the lowest contextual LogVar MAE among the primary comparators, while bootstrap ensembles better match $\vref$ (.548 versus 1.499). Relative to bootstrap member variance, MORSE's score is closer to the contextual reference on all five anchors (paired difference $-.833$, 95\% interval $[-.911,-.757]$) and farther from $\vref$ on all five ($+.951$, $[.881,1.041]$); relative to MC dropout, the corresponding differences are $-.033$ ($[-.041,-.023]$) and $+.314$ ($[.209,.447]$). RN-DER's raw 128-mask variance is closer to the contextual reference than MORSE's raw variance (.615 versus .673; paired difference $+.058$, interval $[.029,.095]$, RN-DER better on 5/5 anchors), and every comparator's predictive mean is closer to the Bayes mean than MORSE's (center RMSE .129--.139 versus .144; RN-DER better on 5/5 anchors). These results distinguish the learned score's efficiency from the quality of the underlying predictive mean.

\begin{table*}[t]
\caption{GP context-distribution diagnostics (five-anchor mean $\pm$ sample SD). A Gaussian moment interval is centered at the four-context average prediction with variance equal to the reported score. Coverage is measured against the 128 Bayes conditional means, not noisy outcomes. Width is relative to their empirical central 95\% interval. Variance ratio is the within-anchor median score/reference ratio. These are operational intervals, not native NIG posteriors.}
\label{tab:gp-context-intervals}
\centering
\small
\setlength{\tabcolsep}{4pt}
\begin{tabular}{lcccc}
\toprule
Method & Variance ratio $\to 1$ & Coverage $\to .95$ & CovGap$_{95}$ $\downarrow$ & Width ratio $\to 1$ \\
\midrule
DER & $8.357 \pm 0.891$ & $0.992 \pm 0.001$ & $0.042 \pm 0.001$ & $2.911 \pm 0.157$ \\
RN-DER & $10.859 \pm 0.799$ & $0.995 \pm 0.001$ & $0.045 \pm 0.001$ & $3.306 \pm 0.121$ \\
Gaussian-MV & $79.267 \pm 8.858$ & $1.000 \pm 0.000$ & $0.050 \pm 0.000$ & $9.250 \pm 0.457$ \\
Deep ensemble (5) & $0.115 \pm 0.021$ & $0.395 \pm 0.033$ & $0.555 \pm 0.033$ & $0.373 \pm 0.028$ \\
Bootstrap ensemble (5) & $0.223 \pm 0.026$ & $0.523 \pm 0.023$ & $0.427 \pm 0.023$ & $0.501 \pm 0.023$ \\
MC dropout (50) & $0.809 \pm 0.083$ & $0.776 \pm 0.024$ & $0.174 \pm 0.024$ & $0.929 \pm 0.047$ \\
\quad no attention (50) & $0.484 \pm 0.055$ & $0.661 \pm 0.017$ & $0.289 \pm 0.017$ & $0.714 \pm 0.049$ \\
MORSE & $1.168 \pm 0.122$ & $0.823 \pm 0.021$ & $0.127 \pm 0.021$ & $1.083 \pm 0.053$ \\
MORSE-IS & $2.545 \pm 0.220$ & $0.936 \pm 0.010$ & $0.014 \pm 0.010$ & $1.597 \pm 0.065$ \\
\bottomrule
\end{tabular}
\end{table*}

\paragraph{Intervals.}
Table~\ref{tab:gp-context-intervals} shows that MORSE's median score-to-reference ratio is 1.17, close to one, but its Gaussian moment intervals cover 82\% of the Bayes conditional means at a nominal 95\%. MORSE-IS nearly reaches nominal coverage by overestimating the variance (ratio 2.5, width ratio 1.6). The evidential baselines overcover with intervals three to nine times too wide, while ensembles underestimate the contextual spread.

\paragraph{Amortization.}
On an independently reserved bank of test-time masks, the raw cross-mask variance of MORSE's predictor has LogVar MAE 1.227 with four masks and that of RN-DER 1.212, compared with .774 for MORSE's learned score averaged over four contexts (four forward passes); the learned score wins on all five anchors against both. The raw MORSE variance first becomes more accurate than the learned score at 32 masked passes on four anchors and at 128 on one; RN-DER requires 16 passes on one anchor and 32 on four. At 128 passes RN-DER's raw variance remains better than MORSE's (.662 versus .716 on this bank). These reserved-bank errors exceed the raw rows of Table~\ref{tab:gp-two-reference}, which reuse masks from the reference bank. These are inference-time pass counts, not a claim of lower training cost.

\paragraph{Observation-policy sweep.}
Fifteen additional MORSE models were trained on the same five anchors at total mask rates .30, .50, and .70 (common-target rates .15, .25, and .35), each scored against its own 128-mask exact reference. The rate-.50 models are retrained with new mask draws, so their error (.787) differs slightly from that of Table~\ref{tab:gp-bayes} (.774). Mean LogVar MAE is 1.227, .787, and .711. On the shared rate-.50 evaluation bank, the mean $\umean$ of these models is .0126, .0154, and .0202. The reported scale therefore depends on the training policy, and so does its accuracy: the rate-.30 models are 56\% less accurate than the rate-.50 models.

\paragraph{Computation.}
On the recorded V100-SXM2-32GB runs, mean training time per anchor was 83.1 seconds for MORSE, 296.2 seconds for the five-member deep ensemble, and 290.1 seconds for the bootstrap ensemble. MC dropout's selected fit took 57.6 seconds plus 159.2 seconds for the three validation candidates. These timers cover training, not inference or earlier method development. The GP trainer uses four masked passes per step for every model, so five-member ensemble training is not compute-matched to a single MORSE fit. Inference uses one, five, or 50 passes per context for MORSE, ensembles, and MC dropout; the four-context averages multiply these counts by four.

\paragraph{Statistical reporting.}
Anchors are the units of replication. We report mean and sample standard deviation, paired differences, and percentile intervals from 20,000 anchor bootstrap draws. The paired MORSE-minus-MC-dropout LogVar difference is $-.0326$ with interval $[-.0413,-.0232]$ and MORSE better on all five anchors; with five anchors the smallest attainable two-sided exact sign-flip $p$-value is .0625, so we make no significance claim. The intervals condition on the fixed test and reference banks. All tables use the same four contexts for every learned score and for $\vref$.

\subsection{Robustness Across GP Configurations}
\label{app:gp-scan}

\paragraph{Settings.}
The audit above is repeated in nine settings that cross Mat\'ern-$5/2$ length scales .025, .05, and .10 with three noise shapes: the step profile of the default setting; a window-homoscedastic profile, constant within each window and equal to that window's average step variance; and a smooth profile with transition width .05 and the same window-average variance. Each setting has five paired anchors of 2,048 training windows, a test bank of 1,024 windows, an oracle reference over 128 masks, and its own independent 20-model $\vref$ bank. The seven estimators, the frozen MORSE settings, and the four-context scoring rule are unchanged, and no loss weight differs across settings. As in the default setting, the oracle conditions on each window's true noise profile. The default row reuses the checkpoints behind Table~\ref{tab:gp-bayes} but scores them against an independent 128-mask oracle bank, which gives 0.779 and 0.809 for MORSE and MC dropout instead of 0.774 and 0.807; the difference reflects the evaluation bank, not retraining. Logged training for the new settings, including MC dropout selection, took 14.9 GPU-hours on V100-SXM2-32GB GPUs.

\begin{table}[!htbp]
\caption{GP configuration scan (five paired anchors per setting; LogVar MAE, mean $\pm$ sample SD; lower is better; $\ell$ is the length scale). (A) Contextual reference; $\Delta$ is MORSE minus MC dropout with a descriptive 95\% interval from 20,000 paired anchor bootstrap draws, and wins count anchors on which MORSE is lower. (B) Training-set reference $\vref$. In every row, MORSE has the lowest mean contextual error and the bootstrap ensemble the lowest mean $\vref$ error among the seven estimators; all seven are listed in the supplement (\texttt{results/gp\_config\_scan/}), and Table~\ref{tab:mc-sites} adds MC dropout without attention-probability dropout. The default row re-scores the checkpoints of Table~\ref{tab:gp-bayes} on an independent oracle bank.}
\label{tab:gp-scan}
\centering
\footnotesize
\setlength{\tabcolsep}{3.5pt}
\begin{tabular}{llccccc}
\multicolumn{7}{l}{(A) Contextual reference}\\
\toprule
$\ell$ & Noise & MORSE & MC dropout & Bootstrap & $\Delta$ [95\% interval] & Wins \\
\midrule
.025 & Step & $0.749 \pm 0.012$ & $1.176 \pm 0.034$ & $2.210 \pm 0.028$ & $-0.427$ $[-0.461, -0.395]$ & 5/5 \\
.025 & Window-homosc. & $0.742 \pm 0.014$ & $1.196 \pm 0.037$ & $2.163 \pm 0.149$ & $-0.454$ $[-0.477, -0.431]$ & 5/5 \\
.025 & Smooth & $0.749 \pm 0.009$ & $1.175 \pm 0.035$ & $2.207 \pm 0.042$ & $-0.426$ $[-0.464, -0.399]$ & 5/5 \\
\midrule
.05 & Step (default) & $0.779 \pm 0.022$ & $0.809 \pm 0.013$ & $1.606 \pm 0.092$ & $-0.029$ $[-0.039, -0.019]$ & 5/5 \\
.05 & Window-homosc. & $0.657 \pm 0.005$ & $0.772 \pm 0.026$ & $1.713 \pm 0.147$ & $-0.114$ $[-0.137, -0.094]$ & 5/5 \\
.05 & Smooth & $0.738 \pm 0.016$ & $0.801 \pm 0.015$ & $1.654 \pm 0.080$ & $-0.063$ $[-0.072, -0.051]$ & 5/5 \\
\midrule
.10 & Step & $0.845 \pm 0.007$ & $0.868 \pm 0.047$ & $1.188 \pm 0.059$ & $-0.023$ $[-0.064, 0.000]^{\dagger}$ & 4/5 \\
.10 & Window-homosc. & $0.589 \pm 0.011$ & $0.682 \pm 0.015$ & $1.424 \pm 0.108$ & $-0.093$ $[-0.110, -0.076]$ & 5/5 \\
.10 & Smooth & $0.776 \pm 0.041$ & $0.801 \pm 0.027$ & $1.250 \pm 0.080$ & $-0.026$ $[-0.057, 0.007]$ & 4/5 \\
\bottomrule
\end{tabular}

\medskip
\begin{tabular}{llccc}
\multicolumn{5}{l}{(B) Training-set reference $\vref$}\\
\toprule
$\ell$ & Noise & MORSE & MC dropout & Bootstrap \\
\midrule
.025 & Step & $1.678 \pm 0.046$ & $0.909 \pm 0.047$ & $0.566 \pm 0.028$ \\
.025 & Window-homosc. & $1.756 \pm 0.050$ & $0.990 \pm 0.038$ & $0.549 \pm 0.022$ \\
.025 & Smooth & $1.669 \pm 0.046$ & $0.908 \pm 0.040$ & $0.564 \pm 0.027$ \\
\midrule
.05 & Step (default) & $1.499 \pm 0.098$ & $1.185 \pm 0.107$ & $0.548 \pm 0.027$ \\
.05 & Window-homosc. & $1.435 \pm 0.081$ & $1.090 \pm 0.074$ & $0.601 \pm 0.029$ \\
.05 & Smooth & $1.503 \pm 0.062$ & $1.176 \pm 0.100$ & $0.557 \pm 0.038$ \\
\midrule
.10 & Step & $1.103 \pm 0.013$ & $0.795 \pm 0.106$ & $0.562 \pm 0.038$ \\
.10 & Window-homosc. & $1.266 \pm 0.055$ & $0.796 \pm 0.028$ & $0.584 \pm 0.041$ \\
.10 & Smooth & $1.255 \pm 0.074$ & $0.919 \pm 0.113$ & $0.581 \pm 0.047$ \\
\bottomrule
\end{tabular}

\smallskip
{\footnotesize $^{\dagger}$The unrounded upper endpoint is $+0.00003$.}
\end{table}

\paragraph{Results.}
Table~\ref{tab:gp-scan} lists MORSE, MC dropout (its closest comparator in the default setting), and the bootstrap ensemble against both references. In all nine settings, MORSE has the lowest mean contextual LogVar MAE among the seven estimators and the bootstrap ensemble the lowest mean LogVar MAE against $\vref$, so the reversal of Table~\ref{tab:gp-bayes} recurs in every setting; averaged over settings, MORSE minus bootstrap is $-0.977$ against the contextual reference and $+0.895$ against $\vref$. The margin over MC dropout depends on the length scale: paired differences are $-0.43$ to $-0.45$ at .025 but $-0.02$ to $-0.11$ at .05 and .10, and the intervals of the step and smooth settings at .10 include zero. The stable advantage is in scale rather than ordering: MC dropout has the higher mean contextual Spearman correlation in five settings, all three at .025 and the window-homoscedastic settings at .05 and .10. Nor does the advantage come from a uniformly better predictor: RN-DER's raw 128-mask response is closer to the oracle reference than MORSE's in seven settings, all except the window-homoscedastic settings at .025 and .10. With five anchors per setting we make no significance claim; the settings share anchor seed roles, so configuration--anchor pairs are not independent replicates.

\subsection{Dropout Sites in MC Dropout}
\label{app:mc-sites}

The MC dropout baseline activates dropout at two attention-probability sites and six activation and residual sites. A second variant sets the attention-probability sites to zero during candidate fitting, refitting, and inference, and keeps the other six. Both variants select $p\in\{.05,.10,.20\}$ by Gaussian NLL on an anchor-only holdout and refit for 100 epochs; the selected $p$ differs between the variants in 22 of 45 cases, so the comparison involves refitting rather than an inference-only switch. Both use 50 draws in each of the four contexts.

\begin{table}[!htbp]
\caption{MC dropout with all dropout sites active (All) and with attention-probability dropout disabled (No attn.) in the nine GP settings (five paired anchors; LogVar MAE, mean $\pm$ sample SD; lower is better; $\ell$ is the length scale). $\Delta$ is No attn.\ minus All with a descriptive 95\% interval from 20,000 paired anchor bootstrap draws. (A) Contextual reference, with MORSE. (B) Training-set reference $\vref$, with the bootstrap ensemble.}
\label{tab:mc-sites}
\centering
\footnotesize
\setlength{\tabcolsep}{3.5pt}
\begin{tabular}{llcccc}
\multicolumn{6}{l}{(A) Contextual reference}\\
\toprule
$\ell$ & Noise & All & No attn. & MORSE & $\Delta$ [95\% interval] \\
\midrule
.025 & Step & $1.176 \pm 0.034$ & $1.914 \pm 0.102$ & $0.749 \pm 0.012$ & $+0.738$ $[+0.625, +0.827]$ \\
.025 & Window-homosc. & $1.196 \pm 0.037$ & $1.883 \pm 0.109$ & $0.742 \pm 0.014$ & $+0.687$ $[+0.608, +0.777]$ \\
.025 & Smooth & $1.175 \pm 0.035$ & $1.909 \pm 0.134$ & $0.749 \pm 0.009$ & $+0.735$ $[+0.593, +0.841]$ \\
\midrule
.05 & Step (default) & $0.809 \pm 0.013$ & $1.012 \pm 0.064$ & $0.779 \pm 0.022$ & $+0.203$ $[+0.151, +0.240]$ \\
.05 & Window-homosc. & $0.772 \pm 0.026$ & $1.052 \pm 0.073$ & $0.657 \pm 0.005$ & $+0.281$ $[+0.208, +0.348]$ \\
.05 & Smooth & $0.801 \pm 0.015$ & $1.041 \pm 0.066$ & $0.738 \pm 0.016$ & $+0.240$ $[+0.188, +0.272]$ \\
\midrule
.10 & Step & $0.868 \pm 0.047$ & $0.912 \pm 0.003$ & $0.845 \pm 0.007$ & $+0.043$ $[-0.0004, +0.068]$ \\
.10 & Window-homosc. & $0.682 \pm 0.015$ & $0.894 \pm 0.056$ & $0.589 \pm 0.011$ & $+0.213$ $[+0.180, +0.269]$ \\
.10 & Smooth & $0.801 \pm 0.027$ & $0.938 \pm 0.028$ & $0.776 \pm 0.041$ & $+0.136$ $[+0.101, +0.168]$ \\
\bottomrule
\end{tabular}

\medskip
\begin{tabular}{llcccc}
\multicolumn{6}{l}{(B) Training-set reference $\vref$}\\
\toprule
$\ell$ & Noise & All & No attn. & Bootstrap & $\Delta$ [95\% interval] \\
\midrule
.025 & Step & $0.909 \pm 0.047$ & $0.453 \pm 0.012$ & $0.566 \pm 0.028$ & $-0.456$ $[-0.487, -0.425]$ \\
.025 & Window-homosc. & $0.990 \pm 0.038$ & $0.454 \pm 0.017$ & $0.549 \pm 0.022$ & $-0.536$ $[-0.575, -0.512]$ \\
.025 & Smooth & $0.908 \pm 0.040$ & $0.459 \pm 0.010$ & $0.564 \pm 0.027$ & $-0.449$ $[-0.474, -0.414]$ \\
\midrule
.05 & Step (default) & $1.185 \pm 0.107$ & $0.732 \pm 0.070$ & $0.548 \pm 0.027$ & $-0.453$ $[-0.542, -0.379]$ \\
.05 & Window-homosc. & $1.090 \pm 0.074$ & $0.663 \pm 0.070$ & $0.601 \pm 0.029$ & $-0.428$ $[-0.523, -0.321]$ \\
.05 & Smooth & $1.176 \pm 0.100$ & $0.690 \pm 0.068$ & $0.557 \pm 0.038$ & $-0.486$ $[-0.573, -0.409]$ \\
\midrule
.10 & Step & $0.795 \pm 0.106$ & $0.599 \pm 0.054$ & $0.562 \pm 0.038$ & $-0.196$ $[-0.326, -0.096]$ \\
.10 & Window-homosc. & $0.796 \pm 0.028$ & $0.598 \pm 0.082$ & $0.584 \pm 0.041$ & $-0.198$ $[-0.271, -0.129]$ \\
.10 & Smooth & $0.919 \pm 0.113$ & $0.606 \pm 0.143$ & $0.581 \pm 0.047$ & $-0.313$ $[-0.467, -0.153]$ \\
\bottomrule
\end{tabular}
\end{table}

Table~\ref{tab:mc-sites} reports both references in the nine settings of Appendix~\ref{app:gp-scan}. Disabling attention-probability dropout raises contextual error and lowers $\vref$ error in all nine settings; in the default setting of Table~\ref{tab:gp-bayes}, the two errors move from .807 and 1.185 to 1.012 and .732. The contextual interval crosses zero only at length scale .10 with step noise. At length scale .025 the variant is the estimator closest to $\vref$, so in every setting the estimator closest to $\vref$ perturbs the model. MORSE has lower contextual error than the variant in every setting, on all five anchors of each. Activation and residual dropout still change later attention maps, so the variant is not a pure parameter-posterior estimator, and its $\vref$ Spearman correlation falls from .797 to .728 in the default setting even as its scale error improves.

\section{Real-Data Protocols and Additional Results}

\subsection{PTB-XL}
\label{app:ptb}

PTB-XL version 1.0.3 contains public 100-Hz, 12-lead electrocardiogram (ECG)
records with patient identifiers~\citep{wagner2020ptbxl,ptbxl103}. We use the five diagnostic
superclasses, official patient-separated folds 1--8 for training, fold 9 for
validation, and fold 10 for testing. Self-supervised pretraining sees only the
training folds. Fine-tuning uses fixed 5\% and 10\% subsets of labeled training
records, with validation Macro-AUROC for checkpoint selection.

Inputs are 12 leads by 1,000 samples and are divided into 40 patches of length
25. The shared encoder has dimension 128, three Transformer layers, four
heads, feed-forward dimension 512, and dropout .10. Self-supervised learning
(SSL) pretraining uses
AdamW ($10^{-3}$ learning rate, $10^{-4}$ weight decay) for 50 epochs;
fine-tuning uses AdamW at $5\times10^{-4}$ for 80 epochs. Batch size is 64,
mixed precision and early stopping are disabled, and the best validation
Macro-AUROC checkpoint is retained. Main results use seeds 1--5 with identical
patient splits and seed-matched labeled subsets.

\begin{table}[!t]
\caption{Paired PTB-XL differences over five matched seeds. Positive values favor MORSE; intervals are paired bootstrap 95\% intervals.}
\label{tab:app-ptb-paired}
\centering
\small
\begin{tabular}{llccc}
\toprule
Labels & Metric & Benefit & Paired 95\% CI & Wins\\
\midrule
\multicolumn{5}{l}{\emph{MORSE versus SimMTM}}\\
5\% & AUROC & .0169 & [.0121, .0208] & 5/5\\
 & AUPRC & .0277 & [.0251, .0308] & 5/5\\
 & Macro-F1 & .0147 & [.0090, .0223] & 5/5\\
 & Brier & .0026 & [$-$.0043, .0096] & 2/5\\
 & AURC & .0009 & [$-$.0078, .0109] & 2/5\\
10\% & AUROC & .0120 & [.0100, .0141] & 5/5\\
 & AUPRC & .0212 & [.0148, .0280] & 5/5\\
 & Macro-F1 & .0160 & [.0135, .0184] & 5/5\\
 & Brier & .0064 & [.0042, .0100] & 5/5\\
 & AURC & .0047 & [.0004, .0105] & 4/5\\
\multicolumn{5}{l}{\emph{ECE, MORSE versus PatchTST}}\\
5\% & ECE & $-$.0123 & [$-$.0206, .0013] & 1/5\\
10\% & ECE & $-$.0107 & [$-$.0295, .0033] & 2/5\\
\bottomrule
\end{tabular}
\end{table}

\begin{table}[!t]
\caption{PTB-XL low-label fine-tuning over five matched seeds (mean $\pm$ sample standard deviation). All methods use the official patient-separated folds, identical seed-specific labeled subsets, and validation Macro-AUROC selection. Bold marks the best mean within each label budget.}
\label{tab:app-ptb-full}
\centering
\small
\setlength{\tabcolsep}{4pt}
\begin{tabular}{lcccccc}
\toprule
Method & AUROC$\uparrow$ & AUPRC$\uparrow$ & Macro-F1$\uparrow$ & Brier$\downarrow$ & ECE$\downarrow$ & AURC$\downarrow$\\
\midrule
\multicolumn{7}{l}{\emph{5\% labeled records}}\\
Supervised & .808$\pm$.005 & .587$\pm$.010 & .578$\pm$.004 & .166$\pm$.003 & .067$\pm$.016 & .120$\pm$.003\\
DAE & .809$\pm$.011 & .588$\pm$.019 & .569$\pm$.017 & .170$\pm$.010 & .111$\pm$.036 & .110$\pm$.010\\
MAE & .835$\pm$.002 & .628$\pm$.012 & .610$\pm$.005 & .152$\pm$.006 & .046$\pm$.021 & .105$\pm$.005\\
PatchTST & .809$\pm$.009 & .587$\pm$.021 & .575$\pm$.007 & .165$\pm$.010 & \textbf{.026}$\pm$.013 & .128$\pm$.014\\
SimMTM & .835$\pm$.005 & .634$\pm$.006 & .607$\pm$.009 & .151$\pm$.005 & .049$\pm$.019 & .098$\pm$.008\\
ST-MEM & .818$\pm$.005 & .612$\pm$.010 & .591$\pm$.007 & .161$\pm$.007 & .049$\pm$.030 & .116$\pm$.010\\
MORSE & \textbf{.852}$\pm$.004 & \textbf{.662}$\pm$.005 & \textbf{.622}$\pm$.004 & \textbf{.148}$\pm$.007 & .038$\pm$.014 & \textbf{.097}$\pm$.007\\
\midrule
\multicolumn{7}{l}{\emph{10\% labeled records}}\\
Supervised & .827$\pm$.005 & .624$\pm$.009 & .592$\pm$.011 & .159$\pm$.011 & .069$\pm$.010 & .110$\pm$.011\\
DAE & .826$\pm$.006 & .613$\pm$.012 & .596$\pm$.012 & .154$\pm$.006 & .079$\pm$.016 & .096$\pm$.008\\
MAE & .842$\pm$.006 & .640$\pm$.015 & .620$\pm$.005 & .148$\pm$.005 & .043$\pm$.020 & .099$\pm$.008\\
PatchTST & .823$\pm$.011 & .612$\pm$.026 & .590$\pm$.015 & .162$\pm$.012 & \textbf{.030}$\pm$.016 & .120$\pm$.014\\
SimMTM & .849$\pm$.005 & .659$\pm$.007 & .622$\pm$.008 & .144$\pm$.005 & .043$\pm$.013 & .090$\pm$.007\\
ST-MEM & .835$\pm$.010 & .637$\pm$.022 & .604$\pm$.012 & .156$\pm$.011 & .052$\pm$.020 & .106$\pm$.012\\
MORSE & \textbf{.861}$\pm$.006 & \textbf{.680}$\pm$.010 & \textbf{.638}$\pm$.007 & \textbf{.138}$\pm$.003 & .041$\pm$.013 & \textbf{.085}$\pm$.003\\
\bottomrule
\end{tabular}
\end{table}

Tables~\ref{tab:app-ptb-paired} and~\ref{tab:app-ptb-full} report both label budgets. MORSE leads AUROC, AUPRC, Macro-F1, Brier score, and AURC at both budgets, while PatchTST retains the lowest ECE. The results therefore support low-label transfer rather than calibration dominance.

\paragraph{Checkpoint provenance.}
The reported PTB-XL numbers are computed from logged per-epoch test metrics at the epoch selected by validation Macro-AUROC. The corresponding selected-epoch weights were not retained; the archived fine-tuned weights are those of the final epoch.

\paragraph{Frozen $U_{\mathrm{rel}}$ readout.}
On a separately retained three-seed checkpoint set, each test record receives the temperature $T_0\exp(\gamma z)$, where $T_0$ is the global temperature, $z$ is the standardized frozen $U_{\mathrm{rel}}$ of the pretrained model, and $\gamma$ is chosen by validation BCE. Compared with the global temperature alone, this lowers ECE from .024 to .016 and AURC from .098 to .087 on clean test records, and ECE from .079 to .040 and AURC from .180 to .177 when 50\% of each record is masked. Because $\gamma$, including its sign, is fitted on validation data, the result shows that $U_{\mathrm{rel}}$ carries sample-wise information; it is not a zero-shot error score, and ranked directly it orders errors inversely (error-AUROC .34 on clean records).

\subsection{IMPTC}
\label{app:imptc}

The Infrastructural Multi-Person Trajectory and Context (IMPTC) dataset,
version 1.0, provides multi-person trajectories and contextual
features~\citep{hetzel2023imptcpaper,hetzel2023imptc}. The balanced
three-class future-event-mode task (FEM3) uses
50-step input windows with 46 motion and context features and predicts three
future-event modes: continuation/free flow, yield/slow/stop, and
directional/crossing/avoidance. The main protocol fine-tunes partially with
30\% labels and evaluates both an in-distribution split and held-out
environment 5.

Windows use horizon 50 and stride 25. The encoder/decoder have three/two
Transformer layers, dimension 128, feed-forward dimension 256, eight heads,
and dropout .10. RAdam uses learning rate $10^{-4}$, batch size 32, maximum
gradient norm 4, and ten SSL/fine-tuning epochs; the learning rate is halved
each epoch. Partial fine-tuning updates the final encoder block and classifier.
A stratified 10\% validation split (seed 2026) selects Macro-F1. Main results
use seeds 1--5. The split audit confirms disjoint source tracks and sequences,
with no overlap between input/future windows across partitions.

Models are selected only by validation Macro-F1. Five-seed in-distribution
means favor MORSE across the displayed task, Brier, and AURC metrics, whereas
the held-out-environment comparison is mixed across MORSE, the masked
autoencoder, and
PatchTST-style. This contrast is treated as a claim boundary: contextual
observability transfers to the trajectory task but does not ensure universal
environment robustness.
The reconstruction baselines are a denoising autoencoder (DAE), a masked
autoencoder (Masked AE), and controlled PatchTST- and SimMTM-style adaptations.

\begin{table*}[!t]
\caption{IMPTC in-distribution Balanced FEM3, 30\% labels and partial
fine-tuning over five seeds (mean $\pm$ sample standard deviation).}
\label{tab:app-imptc-id}
\centering
\small
\setlength{\tabcolsep}{2.5pt}
\resizebox{\linewidth}{!}{%
\begin{tabular}{lrrrrrrr}
\toprule
Method & Accuracy$\uparrow$ & Macro-F1$\uparrow$ & Balanced$\uparrow$ &
Directional F1$\uparrow$ & Brier$\downarrow$ & ECE$\downarrow$ &
AURC$\downarrow$\\
\midrule
Supervised & .7154$\pm$.0029 & .6362$\pm$.0178 & .6479$\pm$.0109 &
.2951$\pm$.0552 & .3751$\pm$.0009 & .0243$\pm$.0074 & .1435$\pm$.0010\\
DAE & .7240$\pm$.0027 & .6565$\pm$.0074 & .6622$\pm$.0052 &
.3392$\pm$.0215 & .3539$\pm$.0019 & .0211$\pm$.0007 & .1309$\pm$.0016\\
Masked AE & .7304$\pm$.0027 & .6645$\pm$.0049 & .6693$\pm$.0032 &
.3539$\pm$.0161 & .3459$\pm$.0026 & \textbf{.0171}$\pm$.0034 &
.1284$\pm$.0030\\
PatchTST-style & .7363$\pm$.0015 & .6742$\pm$.0150 & .6783$\pm$.0102 &
.3697$\pm$.0503 & .3346$\pm$.0059 & .0244$\pm$.0064 & .1193$\pm$.0049\\
SimMTM-style & .7338$\pm$.0018 & .6563$\pm$.0049 & .6655$\pm$.0030 &
.3152$\pm$.0175 & .3414$\pm$.0024 & .0248$\pm$.0056 & .1244$\pm$.0020\\
\textbf{MORSE} & \textbf{.7407}$\pm$.0018 & \textbf{.6882}$\pm$.0087 &
\textbf{.6878}$\pm$.0066 & \textbf{.4084}$\pm$.0339 &
\textbf{.3273}$\pm$.0036 & .0319$\pm$.0038 &
\textbf{.1081}$\pm$.0034\\
\bottomrule
\end{tabular}%
}
\end{table*}

\begin{table*}[!t]
\caption{IMPTC held-out environment 5 under the same protocol (five seeds,
mean $\pm$ sample standard deviation). The mixed ranking is an explicit boundary on
the paper's out-of-distribution claim.}
\label{tab:app-imptc-ood}
\centering
\small
\setlength{\tabcolsep}{2.5pt}
\resizebox{\linewidth}{!}{%
\begin{tabular}{lrrrrrrr}
\toprule
Method & Accuracy$\uparrow$ & Macro-F1$\uparrow$ & Balanced$\uparrow$ &
Directional F1$\uparrow$ & Brier$\downarrow$ & ECE$\downarrow$ &
AURC$\downarrow$\\
\midrule
Supervised & .6194$\pm$.0495 & .4714$\pm$.0422 & .5352$\pm$.0430 &
.0030$\pm$.0046 & .4956$\pm$.0657 & .0435$\pm$.0175 & .2688$\pm$.0781\\
DAE & .7134$\pm$.0029 & .6022$\pm$.0149 & .6311$\pm$.0059 &
.1711$\pm$.0506 & .3822$\pm$.0025 & \textbf{.0212}$\pm$.0038 &
.1621$\pm$.0044\\
Masked AE & .7199$\pm$.0032 & \textbf{.6202}$\pm$.0140 & .6418$\pm$.0067 &
\textbf{.2196}$\pm$.0455 & .3731$\pm$.0023 & .0237$\pm$.0099 &
.1569$\pm$.0029\\
PatchTST-style & \textbf{.7256}$\pm$.0052 & .6158$\pm$.0095 &
.6418$\pm$.0068 & .1899$\pm$.0264 & \textbf{.3708}$\pm$.0077 &
.0249$\pm$.0088 & \textbf{.1568}$\pm$.0084\\
SimMTM-style & .7198$\pm$.0025 & .6022$\pm$.0090 & .6335$\pm$.0041 &
.1570$\pm$.0264 & .3807$\pm$.0024 & .0346$\pm$.0066 & .1656$\pm$.0032\\
\textbf{MORSE} & .7225$\pm$.0084 & .6129$\pm$.0122 &
\textbf{.6425}$\pm$.0078 & .1891$\pm$.0203 & .3761$\pm$.0131 &
.0411$\pm$.0224 & .1595$\pm$.0126\\
\bottomrule
\end{tabular}%
}
\end{table*}

\begin{table*}[!t]
\caption{Matched-seed IMPTC comparisons with the two closest reconstruction
baselines. Positive $\Delta$ favors MORSE. Intervals are deterministic paired
bootstrap 95\% CIs (20,000 replicates); $d_z$ is the paired standardized
effect. With five pairs, the minimum two-sided exact sign-test $p$ is .0625;
intervals, wins, and effects are reported without significance claims.}
\label{tab:app-imptc-paired}
\centering
\small
\setlength{\tabcolsep}{4.2pt}
\renewcommand{\arraystretch}{0.88}
\begin{tabular}{lllrrr}
\toprule
Split & Baseline & Metric & Oriented $\Delta$ [95\% CI] & Wins/5 & $d_z$\\
\midrule
In-dist. & Masked AE & Macro-F1 & $.0237\ [.0154,.0313]$ & 5 & 2.30\\
In-dist. & Masked AE & Brier & $.0186\ [.0144,.0228]$ & 5 & 3.41\\
In-dist. & Masked AE & ECE & $-.0147\ [-.0179,-.0111]$ & 0 & $-3.46$\\
In-dist. & Masked AE & AURC & $.0203\ [.0159,.0247]$ & 5 & 3.62\\
In-dist. & PatchTST-style & Macro-F1 & $.0140\ [.0078,.0218]$ & 5 & 1.57\\
In-dist. & PatchTST-style & Brier & $.0073\ [.0005,.0140]$ & 3 & .85\\
In-dist. & PatchTST-style & ECE & $-.0075\ [-.0129,-.0021]$ & 0 & $-1.03$\\
In-dist. & PatchTST-style & AURC & $.0112\ [.0055,.0169]$ & 5 & 1.54\\
\midrule
Held-out env. 5 & Masked AE & Macro-F1 & $-.0072\ [-.0211,.0059]$ & 2 & $-.41$\\
Held-out env. 5 & Masked AE & Brier & $-.0031\ [-.0130,.0079]$ & 2 & $-.23$\\
Held-out env. 5 & Masked AE & ECE & $-.0174\ [-.0412,.0043]$ & 2 & $-.58$\\
Held-out env. 5 & Masked AE & AURC & $-.0026\ [-.0117,.0081]$ & 2 & $-.20$\\
Held-out env. 5 & PatchTST-style & Macro-F1 & $-.0028\ [-.0133,.0088]$ & 2 & $-.20$\\
Held-out env. 5 & PatchTST-style & Brier & $-.0053\ [-.0143,.0033]$ & 1 & $-.46$\\
Held-out env. 5 & PatchTST-style & ECE & $-.0163\ [-.0383,.0010]$ & 1 & $-.65$\\
Held-out env. 5 & PatchTST-style & AURC & $-.0027\ [-.0104,.0049]$ & 2 & $-.27$\\
\bottomrule
\end{tabular}
\end{table*}

The paired audit sharpens the aggregate interpretation. MORSE improves the
in-distribution Macro-F1, Brier, and AURC over the masked autoencoder on all
five paired seeds,
but its in-distribution ECE is
worse on all five. The held-out-environment intervals cross zero for the task,
Brier, and AURC comparisons, so that split supports no uniform
out-of-distribution advantage. Table~\ref{tab:app-imptc-missing} masks a random
fraction of feature tokens at evaluation on the same held-out environment.
Accuracy and both selective-risk scores degrade beyond 25\% missingness, and
the native score $\urel$ ranks errors worse than softmax confidence at every
level, so the reconstruction factor is not a substitute for task confidence.

\begin{table*}[!t]
\begin{minipage}[t]{.49\textwidth}
\caption{IMPTC split-integrity audit. Windows include the 50-step observation
and future-label horizons; all overlaps are computed before training.}
\label{tab:app-imptc-split-audit}
\centering
\small
\setlength{\tabcolsep}{3pt}
\resizebox{\linewidth}{!}{%
\begin{tabular}{lrrr}
\toprule
Partition & Windows & Tracks & Sequences\\
\midrule
Train & 52,533 & 1,925 & 189\\
Validation & 9,575 & 365 & 40\\
Test & 10,437 & 375 & 41\\
\midrule
Pair & \shortstack{Track\\overlap} & \shortstack{Sequence\\overlap} &
\shortstack{Full-window\\overlap}\\
Train--validation & 0 & 0 & 0\\
Train--test & 0 & 0 & 0\\
Validation--test & 0 & 0 & 0\\
\bottomrule
\end{tabular}%
}
\end{minipage}\hfill%
\begin{minipage}[t]{.49\textwidth}
\caption{MORSE under random feature-token missingness at evaluation
(held-out environment 5, 30\% labels, partial fine-tuning; mean of three
seeds). Conf.\ AURC defers by softmax confidence and native AURC by the
relative factor $\urel$.}
\label{tab:app-imptc-missing}
\centering
\small
\setlength{\tabcolsep}{3pt}
\begin{tabular}{rrrrrr}
\toprule
Missing & Accuracy & F1 & ECE & \shortstack{Conf.\\AURC} &
\shortstack{Native\\AURC}\\
\midrule
0 & .7269 & .6166 & .0323 & .1519 & .2035\\
.10 & .7249 & .6098 & .0422 & .1531 & .2056\\
.25 & .7167 & .5972 & .0518 & .1700 & .2157\\
.50 & .5928 & .4421 & .0912 & .3176 & .3998\\
.75 & .4835 & .2409 & .1899 & .5148 & .5223\\
\bottomrule
\end{tabular}
\end{minipage}
\end{table*}

\subsection{Metric and Selection Conventions}

AUROC, area under the precision--recall curve (AUPRC), Macro-F1, balanced
accuracy, and classwise F1 measure predictive transfer. Brier score measures
probabilistic accuracy, ECE summarizes binned
calibration, and AURC integrates error risk as low-confidence examples are
deferred~\citep{brier1950verification,guo2017calibration,geifman2017selective}.
Lower AURC is better. No test metric is used for checkpoint selection or
temperature fitting.

Main tables report mean and sample standard deviation when seed-level outputs are
available. Matched methods should additionally be compared through paired
differences, bootstrap confidence intervals, and effect sizes. With five
real-data seeds, intervals and consistency of paired effects are more
informative than a binary significance label.

\section{Reproducibility and Artifact Map}
\label{app:repro}

\subsection{Runs and Selection}

The fixed and parameterized pointwise diagnostics contain 50 fits/tasks. Main
window mechanism studies contain ten paired seeds. The strict empirical
reference contains 50 first-order models and five paired second-order anchor
datasets. The residual-coupling control uses 20 paired seeds. PTB-XL and IMPTC
main tables use five seeds. The GP audit uses five paired anchors per method, a
20-model reference bank, and fifteen additional models for the policy sweep; the
configuration scan (Appendix~\ref{app:gp-scan}) covers nine settings, each with
five paired anchors and its own 20-model reference bank, and reuses the default
checkpoints; the MC dropout-site study adds 45 final models from 180 candidate
and refit fits, and the matched component study 300 fits; the pulse proxy
reference uses twelve proxy models.
PTB-XL checkpoints are selected by validation Macro-AUROC and IMPTC
checkpoints by validation Macro-F1.

\subsection{Frozen Configuration and Seed Axes}

The machine-readable freeze is
\url{configs/protocols.json} in the supplementary archive.
It records the final objective, dataset-specific settings, every reported seed
axis, and SHA-256 hashes of implementation and result files. The compact seed
ledger is shown below; model seeds and evaluation seeds are deliberately kept
separate.

\begin{table*}[!t]
\caption{Frozen seed axes for reported experiments. Ranges are inclusive.}
\label{tab:app-seeds}
\centering
\small
\begin{tabular}{>{\raggedright\arraybackslash}p{.24\textwidth}>{\raggedright\arraybackslash}p{.30\textwidth}>{\raggedright\arraybackslash}p{.34\textwidth}}
\toprule
Experiment & Model/data seeds & Fixed auxiliary seeds\\
\midrule
Fixed and parameterized one-dimensional tasks & 10000--10049 & fixed-generator settings in merge manifests\\
Pulse-window comparison/\allowbreak components/$K$ & 1--10 & paired masks and data within each seed\\
MORSE versus MORSE-RC & 1--20 & paired data, initialization, and masks\\
Strict five-anchor reference & anchors 900001--900005; reference 1--50 & test 424242; normalization 777777; init/mask/eval/loader 12001--12004\\
Strict $N$ sensitivity & anchors 900001--900010; reference 1--50 at each $N$ & same 424242, 777777, and 12001--12004\\
Pulse proxy reference & standard 700011--700018; large 700021--700024 & 128-mask bank 424243\\
GP audit & anchors 830001--830005; reference 830011--830030; ensembles 830031--830080 & test 830101; calibration 830000\\
GP configuration scan & the GP audit seeds in every setting & oracle masks 971000; reserved masks 971001; interval bootstrap 971002\\
MC dropout sites & the GP audit seeds in every setting & holdout 820000+anchor; draws from 820001; interval bootstrap 971002\\
Component study, GP & anchors 830001--830005; reference 830011--830030 & normalization 830000; test 830101; init/loader/dropout 12001/12004/12005; intervals 730001\\
Component study, pulse & seeds 1--10 & training masks 220001+epoch; response bank 424271\\
PTB-XL low-label & 1--5 & official folds; matched label subsets\\
IMPTC Balanced FEM3 & 1--5 & validation/split seed 2026\\
\bottomrule
\end{tabular}
\end{table*}

\subsection{Data and Split Integrity}

PTB-XL version 1.0.3 is used under CC BY 4.0. IMPTC version 1.0 is obtained
from Zenodo record 14811016 under CC BY-NC 4.0. Synthetic data are generated
entirely by supplied code. Raw external datasets and model checkpoints are not
redistributed. Reproduction scripts should verify PTB-XL patient separation and
IMPTC sequence/track/window separation before training. The validation-selected
PTB-XL fine-tuning weights were not retained (Appendix~\ref{app:ptb}).
The code base includes utilities adapted from Time-Series-Library (MIT
license), some of which that library adapted from other projects;
\texttt{THIRD\_PARTY\_NOTICES.md} in the supplement lists the files and
licenses.

\subsection{Environment and Compute}

The recovered environment uses Python 3.12.9, PyTorch 2.6.0+cu124, NVIDIA
Compute Unified Device Architecture (CUDA) 12.4,
NumPy 2.3.4, SciPy 1.16.3, scikit-learn 1.6.1, pandas 2.3.3, and Matplotlib
3.10.1. Runs use single NVIDIA A100- or V100-family graphics processing units (GPUs)
and no mixed precision; the GP audit and part of the checkpoint retraining ran on
V100-SXM2-32GB GPUs. The exact central processing unit (CPU)
model, accelerator-memory variant, driver revision, and total project-wide
compute were not preserved in the frozen logs and are therefore not used for
any efficiency claim. The measured $K$-view window costs are reported in
Table~\ref{tab:app-kview} and the GP training times in Appendix~\ref{app:gp}. The matched component study took 6.6 and the MC dropout-site study 3.2 V100 GPU-hours of logged training.

\subsection{Claim-to-Evidence Map}

Table~\ref{tab:app-claim-map} maps each claim to its primary evidence and to
the boundary it requires.

\begin{table}[!t]
\caption{Claim-to-evidence map for this paper.}
\label{tab:app-claim-map}
\centering
\small
\begin{tabular}{>{\raggedright\arraybackslash}p{0.28\textwidth}>{\raggedright\arraybackslash}p{0.30\textwidth}>{\raggedright\arraybackslash}p{0.32\textwidth}}
\toprule
Claim & Primary evidence & Scope \\
\midrule
Marginal fitting does not identify the relative allocation
& Proposition~\ref{prop:allocation}
& Concerns evidence from the marginal, not uniqueness under arbitrary regularization \\
Cross-context mean variance is observable and has a population limit
& Propositions~\ref{prop:between-context} and~\ref{prop:population}
& Policy-relative; not a parameter posterior \\
The learned score aligns with the contextual reference
& GP audit, Tables~\ref{tab:gp-bayes}, \ref{tab:gp-two-reference}, \ref{tab:gp-context-intervals}, and~\ref{tab:gp-scan}
& Four-context score; scale error in nine GP settings, five anchors each, not ordering (MC dropout has the higher Spearman correlation in five); MC dropout close at length scale .10; undercoverage; benchmark-dependent (Table~\ref{tab:pulse-extended}) \\
The learned score costs fewer inference passes than direct repeated masking of comparable accuracy
& GP amortization audit, Appendix~\ref{app:gp}
& Default setting; four passes versus 16 to 128; inference only, not training cost \\
Contextual and training-set variation are distinct
& GP two-reference audit in nine settings (Table~\ref{tab:gp-scan}); Proposition~\ref{prop:population}
& MORSE--bootstrap reversal in all nine GP settings; disabling attention dropout moves MC dropout toward $\vref$ (Table~\ref{tab:mc-sites}), by refitting rather than an inference-only switch; not observed on the pulse benchmark \\
The predictor's contextual response converges with $N$
& Pulse training-size audit, Table~\ref{tab:pulse-n}
& Approximate proxy reference, not exact Bayes \\
MORSE pretraining improves low-label transfer
& PTB-XL 5\%/10\% tables and IMPTC in-distribution table
& Controlled baselines; no calibration or out-of-distribution claim; selected-epoch weights not retained \\
Bound, decoupling, and normalization drive scale alignment; the head is interchangeable
& Matched component and head study, Appendix~\ref{app:components}
& Fixed-boundary GP and a new pulse cohort; one initializer; Condition and Order help fitting and ranking, not scale; NIG improves pulse selective metrics only \\
\bottomrule
\end{tabular}
\end{table}

\subsection{Artifact Validation}

The supplementary code archive contains the final MORSE implementation, the synthetic generators, the PTB-XL and IMPTC preprocessing, launchers, aggregation code, seed-level results, and table exporters. Its validator (\texttt{scripts/validate\_reproducibility\_package.py}) checks expected method sets and seed counts, shared hashes for paired anchors, disjoint anchor and reference seeds, and checksums for the retained artifacts. \texttt{scripts/reproduce\_gp.py} recomputes Table~\ref{tab:gp-bayes} and every Figure~\ref{fig:gp-two-references} bin on a CPU from the released scores. The recorded GP validation (\texttt{results/gp/validation.json}) covers 227 input checksums, 90 training manifests, and 60 paired rows.

\subsection{Claim Summary}

MORSE learns a score grounded in policy-relative observability: how strongly a prediction depends on its visible context. Its contribution is to replace an unobserved semantic allocation with a controlled behavioral target whose assumptions can be tested.

\end{document}